%% file: main_arxiv.tex
\documentclass[11pt]{article}
\usepackage[a4paper, top=2.0cm, bottom=2.0cm, left=3cm, right=3cm]{geometry}
\usepackage{amsmath, amssymb, graphicx, url, tcolorbox, hyperref,enumitem}
\usepackage{xcolor}

\definecolor{navy}{RGB}{0, 0, 128}
\usepackage[numbers]{natbib}   % for
\usepackage{footnote}
\makesavenoteenv{tcolorbox}
\usepackage{color,soul}
\usepackage{graphicx}
\usepackage{subcaption}
\usepackage{booktabs}
\usepackage{subcaption} % in preamble
\usepackage{cleveref}
\usepackage{xurl,amsthm}

\theoremstyle{plain} 
\newtheorem{proposition}{Proposition}

\title{Evaluation without Generation: \\ Non-Generative Assessment of Harmful Model Specialization with Applications to CSAM}

\author{
Vinith M. Suriyakumar$^{1}$\thanks{Corresponding authors: vinithms@mit.edu, ashia07@mit.edu},~ 
Ayush Sekhari$^{2}$\thanks{These authors contributed equally.},~ 
Lena Stempfle$^{1}$\footnotemark[2],~ Robertson Wang$^{3}$\footnotemark[2], \\  Michael Simpson$^{3}$,~ 
Rebecca Portnoff$^{3}$,~ 
Marzyeh Ghassemi$^{1}$,~
Ashia C. Wilson$^{1}$\footnotemark[1]
\and
\text{$^{1}$MIT, $^{2}$BU, $^{3}$Thorn} 
}

\date{}

\begin{document}

\maketitle
\thispagestyle{empty}
\begin{abstract}
Auditing the fine-tunes of open-weight generative models for harmful specialization has become a new governance challenge for model hosting platforms. The standard toolkit, \textit{generative evaluation} via curated prompts or red-teaming, does not scale to platform-level auditing and breaks down entirely for domains like child sexual abuse material (CSAM) where generation is legally constrained. This motivates the {\em Evaluation without Generation} problem: assessing model capabilities without producing outputs. In such settings, capability must be inferred from the model's state, either its parameters or internal representations, rather than its outputs. We introduce {\em Gaussian probing}, a method that characterizes how LoRA adaptors functionally perturb a model by measuring its internal responses to a reference ensemble of Gaussian latent states. Unlike raw-weight baselines, Gaussian probing reliably distinguishes benign from harmful specialization without sampling outputs.
We demonstrate effectiveness in high-risk domains, including detecting models specialized for CSAM under realistic constraints. Our results show that Gaussian probing provides a scalable non-generative alternative for evaluating high-risk generative systems and remains robust to weight rescaling, a representative adversarial manipulation.
\end{abstract}

\section{Introduction}
\input{intro_option2}

\section{Problem Formulation}
\label{sec:problem}
\input{capability}

\section{Evaluating Gaussian Probing Against the Desiderata}
\label{Sec:Explore}

\input{sections/adult_sexual_content_detection}

\section{Detecting CSAM Specialization in the Wild}
\label{Sec:Wild}

\input{sections/child_sexual_abuse_content_detection}

\section{Related Work}

\paragraph{Standard safety evaluation for image generation models} Safety evaluation of open-weight generative models typically relies on a set of tools similar to those used for closed-weight generative models. Relying on prompt-based testing~\cite{lee2023holistic, khader2024diffguard, luccioni2023stable} and red teaming~\cite{ganguli2022red, li2025diffuguard, qu2023unsafe} where a large number of prompts are produced, and their output is generated to measure the frequency with which these models generate harmful content. There have been many documented failure modes of generative evaluation, jailbreaking~\cite{jin2025jailbreakdiffbench, wei2023jailbroken, ma2025jailbreaking, yang2024sneakyprompt} and finetuning~\cite{truong2025attacks, zheng2023understanding, li2024shake} attacks being the most notable. Jailbreaking involves either the automated or manual construction of prompts that bypass any safety training that the model has undergone. Finetuning attacks involve finetuning the model on unrelated or related data to undo safety training. Both have also become standard evaluation techniques to assess the efficacy of safety training against adversarial attacks.

\paragraph{Safety interventions for open-weight model deployment} Due to the safety vulnerabilities that have surfaced for open-weight image generation models, there has been a push for developing interventions to address these threats. The main challenge for open-weight models is that input and output classifiers can be easily removed~\cite{rando2022red}. To be effective, interventions need to modify the model itself in ways that are much harder to undo. Concept unlearning~\cite{wu2024erasediff, fuchi2024erasing, zhang2024defensiveunlearningadversarialtraining, ko2024boosting, lu2024mace, fan2023salun, wu2024scissorhands, huang2024recelerreliableconcepterasing, gandikota2023erasing, gandikota2024unified}, a form of post-training where the weights of the models are modified to remove a concept, has become the primary technique of choice for improving the safety of open-weight models. Broadly, these techniques all develop an objective function that represents what the underlying model would be, had it not been trained on the harmful concept, and aims to steer the weights towards this idealized version. CivitAI currently uses a form of concept unlearning to prevent CSAM from being generated in its web-based content generation tool~\cite{lyu2024, civitaiSPM}. Nevertheless, they must still rely on user or external notifications to discover harmful models, so it does not provide a pre-distribution safeguard. In addition, there have been numerous studies showing weaknesses in concept unlearning as a safeguard for open-weight models~\cite{suriyakumar2024unstable, gao2025meta, lu2025concepts}. Overall, concept unlearning remains a promising tool for addressing safety post-distribution on platforms, but does not tackle the core governance channel of our work on pre-distribution screening. While there has been one related work attempting to address evaluation without generation they make assumptions which do not satisfy our problem setting and the motivating application. Specifically, the study introduces prompt agnostic image-free auditing (PAIA)~\cite{yuan2025lurks} which aims to formalize and address evaluation without generation. Yet they require using a small set of images from the target concept that the LoRA was specialized on. This is unlawful in the CSAM domain. Our work formalizes the problem at hand in a more applicable form for AIG-CSAM, proposes an algorithm which only operates on the weights, and shows success in identifying in the wild CSAM specialization.

\paragraph{Weight-space learning} There is a growing line of research investigating model parameters as a data modality, called \textit{weight-space learning}. We lean on this research to construct our baselines and to inspire our idea for Gaussian probing. Neural network weights have been shown to encode information about training data, task structure, and fine-tuning objectives, motivating weight-space learning for model auditing without content generation~\citep{Schurholt2021, unterthiner2021predictingneuralnetworkaccuracy, schurholt2024towards}. Prior work includes hyper-representation learning on neural network weights~\citep{Schurholt2021}, predicting accuracy from weight statistics~\citep{unterthiner2021predictingneuralnetworkaccuracy}, sequential transformer embeddings of weights~\citep{schurholt2024towards}, interpreting customized diffusion models~\citep{Dravid2024Interpreting}, and Deep Linear Probe Generators~\citep{kahana2024deeplinearprobegenerators}. 

Broadly, these approaches can be categorized into mechanistic approaches and functional approaches. Mechanistic approaches typically train a model to learn an embedding of the weights that are then used for classification downstream. At the current moment, most of these methods are simply not effective at the parameter scale of diffusion models, especially with the small sample size of CSAM-specialized LoRAs. Additionally, some make assumptions, such as that the LoRAs modify the same weights across all finetunes. These assumptions are not realistic for the distribution of LoRAs observed in the wild. 

Functional approaches usually start with random probes, as in our technique, but then aim to optimize the probes directly to maximize classification performance. We choose not to explore these probe optimization approaches because they may inadvertently lead to probes that would generate CSAM, and we recommend proceeding with caution when exploring these approaches, since they may constitute \textit{intent} to generate the content, which may be unlawful for CSAM. 

Related advances in mechanistic interpretability further show that internal model components can be systematically analyzed to reveal functional structure and learned representations~\citep{bereska2024mechanisticinterpretabilityaisafety, conmy2023towards}. However, these approaches are not designed to detect specific harmful capabilities or to operate in settings where output generation is prohibited, and there is currently no scalable method to assess whether a model has been specialized for harmful content without generating outputs.

\paragraph{Addressing AIG-CSAM in Generative Models} There has been a limited amount of research and guidance written about addressing AIG-CSAM in generative models. The main source of guidance thus far has been from a collaborative effort spearheaded by Thorn on \textit{Safety by Design}~\cite{thornsafetybydesign, thornprogressreport}. This initiative outlined the different touchpoints in the generative AI lifecycle where AIG-CSAM concerns can arise, the different technical affordances that could be useful to address these concerns, and the status of these across participating organizations. Some of the technical affordances discussed include training data filtering, unlearning, and measuring capabilities via latent spaces without prompting. This work directly addresses the latter. Initial research has started to study the efficacy of potential interventions proposed in this initiative. So far, it has been demonstrated that training data filtering is not fully effective in preventing AIG-CSAM because of concept fusion of benign child generation and adult sexual content capabilities~\cite{cretu2025evaluating}.
At the same time, several open challenges in AI safety have been identified that are relevant for child safety \cite{openproblemschildsafety2026}. Our work is the first to our knowledge to successfully build a method that effectively addresses one of open problems, namely evaluation restrictions on assessing AIG-CSAM capabilities. Nonetheless, additional open problems remain that complicate the development of effective safety mechanisms for AIG-CSAM.

\section{Discussion and Conclusion}
\input{sections/discussion}

\section*{Ethical Considerations}
\label{sec:ethics}

This work presents a new paradigm for assessing the extent to which image generation models have been finetuned on harmful content, specifically CSAM, without any content generation. In this research process and its communication in this paper, numerous steps were taken to conduct it ethically.

First, in the curation of adult sexual content from the internet, we ran Thorn’s Safer detection technology over the image datasets and removed any images flagged as potentially containing CSAM. However, we recognize that it is still possible for adult non-consensual intimate imagery (NCII) to exist in the dataset we used for training and the re-victimization harm that may arise from this. Any presence of such material, regardless of scale, can cause trauma and must not be minimized. We believe that research such as this is critical to building scalable mechanisms by which to prevent the proliferation of models specialized for the creation of CSAM and NCII. We want to highlight that this work is another example of why research into adult content datasets collected consensually are essential. Upon publication of this paper, we will be deleting all of the adult content datasets and LoRAs specialized on adult content.

Second, we intentionally omitted and abstracted significant details about the LoRAs specialized for CSAM generation on which we evaluate our methodology. We do not disclose details on their whereabouts, verification, names, or how they were obtained. In doing so, we ensure that no information is provided that would enable individuals to locate these models beyond what is publicly available, thereby reducing misuse risk and preserving operational security. We affirm that we never possessed, accessed, attempted to, or generated CSAM. We also did not train any LoRAs on CSAM ourselves or possess CSAM specialized LoRAs. All such data remained solely with authorized entities, and all testing was conducted in compliance with applicable laws and organizational safeguards. In order to justify our approach and to show its efficacy in detection, we have evaluated the approach on non-CSAM tasks, as demonstrated in the experimental section of this paper.

Third, we believe that publication of these results is beneficial to the AI safety ecosystem as a whole and will help prevent additional harm. By demonstrating a method that can assess abuse capabilities in generative models without having to generate illicit content, we provide model hosting platforms, AI developers, policy makers, and safety researchers with a legally viable tool for scalable proactive detection of abuse enabled generative models. We aim to support continued innovation while equipping stakeholders with technology to detect and curb the proliferation of such models.

\section*{Responsible Disclosure System}

At the time of release of this work, there may be organizations that start to deploy this method in partnership with Thorn. As such, we want to encourage responsible disclosure of vulnerabilities of this system. Given the high-risk nature of CSAM, as a community we should aim to prevent providing individuals with techniques to obfuscate these detection systems before they have been addressed. At the same time, we view follow-up research to improve and identify vulnerabilities in this system as essential to making non-generative auditing robust. Therefore, we strongly encourage researchers to responsibly disclose vulnerabilities to our team at \url{ewg@mit.edu}. Our team will engage with reports in a timely manner and work with you to patch the vulnerability. After the patch has been deployed to all current users, we encourage publication of the vulnerability to advance our collective knowledge of the field. We encourage researchers to submit vulnerability disclosures with the following included:

\begin{itemize}
    \item A detailed description of the vulnerability, including how it was discovered, its immediate impact, and whether an immediate patch has been identified
    \item A proof-of-concept notebook demonstrating how the vulnerability can be exploited
    \item Information about plans for public disclosure including timelines
\end{itemize}

\section*{Acknowledgements}

VMS was supported by a Bridgewater AIA Labs Research Fellowship. We thank MIT ORCD for providing technical support and managing the GPU clusters used to perform our LoRA training and probe extraction.

\newpage

\bibliography{references}
\bibliographystyle{plainnat}

\newpage
\appendix

\section{Consistency of Gaussian Probing}
\label{App:Theory}
The population object underlying Gaussian probing is given by~\eqref{eq:pop_probe}, and the empirical representation~\eqref{eq:emp_probe} is its Monte Carlo approximation. The following result shows that this approximation is statistically well behaved and clarifies the condition under which it can support correct classification.

\begin{proposition}[Consistency of Gaussian probing]
\label{prop:gp_consistency}
Suppose that $\bar{H}(\Delta;\nu)$ has finite second moment for every 
fixed adaptor $\Delta$. Then
\[
\widehat{\Psi}_m(\Delta)
\;\xrightarrow[m\to\infty]{\mathrm{a.s.}}\;
\Psi(\Delta).
\]
If, in addition, the specialization classes are separated by a positive 
margin in $\Psi$-space, then any margin-respecting plug-in classifier 
built on $\widehat{\Psi}_m(\Delta)$ is asymptotically correct.
\end{proposition}

\noindent{\em Proof Sketch.}
The convergence $\widehat{\Psi}_m(\Delta) \to \Psi(\Delta)$ follows from the strong law of large numbers, since the probes are sampled independently from the Gaussian reference distribution. If the classes are separated in the population representation, then for sufficiently large $m$ the empirical representation lies in the same decision region as the population representation, yielding asymptotically correct 
classification.$\qed$

\section{Probe Classifier Implementation}

\paragraph{Extraction} For extracting the probes from the models we fix the internal randomness using a specific random seed across all LoRAs. This is important so that all of the representations we extracted start from the same initial Gaussian noise. Throughout our experiments in Section~\ref{Sec:Explore} we sample 1024 probes unless stated otherwise. In our experiments in Section~\ref{Sec:Wild} we sample 512 probes for SD 1.5 and SDXL 1.0. Due to the significantly larger size of FLUX.1-dev and computational restrictions when interacting with CSAM-specialized LoRAs for operational security, we only sample 4 probes. We do so over 30 timesteps. So each LoRA is mapped to a tensor in the shape of (N, T, D) where N is the number of probes, T is the number of timesteps, and D is the dimension of the representation we extract. We do so by simply implementing hooks in that capture the representations during forward passes up to the layers of interest. To prevent randomness due to the solver introducing noise, we use a deterministic solver for each architecture when extracting the probes. Specifically, for SD1.5 and SDXL 1.0, which are denoising diffusion implicit models (DDIM) that use a deterministic ODE instead of an SDE. For FLUX.1-dev, the default flow-matching scheduler with dynamic shift is deterministic. 

We note that probe tensors are not guaranteed to be numerically identical across hardware / micro-batch configurations due to bf16 non-determinism. However, this is a non-issue given our theoretical intuition for the method because the noise from non-determinism is class-agnostic and the classifier is trained to distinguish distributional moments rather than exact values. We empirically validate this on all three architectures across classes by training linear classifiers to distinguish between our averaged probes on different subsets of the entire probe set. We find that across all three architectures, classes, and all layer choices, classifiers are unable to distinguish between averaged probes from different initializations (Figures~\ref{fig:subset_dist_sd15}, \ref{fig:subset_dist_sdxl}, \ref{fig:subset_dist_flux1dev}).

\paragraph{Classifier} For the classifier we split the dataset of probes based on the LoRAs into a 70\%, 10\%, 20\% train, validation, and test train split. We train a linear classifier for 200 epochs, with batch size 32, and learning rate $1^{-3}$. We take an argmax over the softmax probabilities to assign the class between SFW, NSFW, and CSAM.

\begin{figure}[h]
    \centering
    \includegraphics[width=0.8\linewidth]{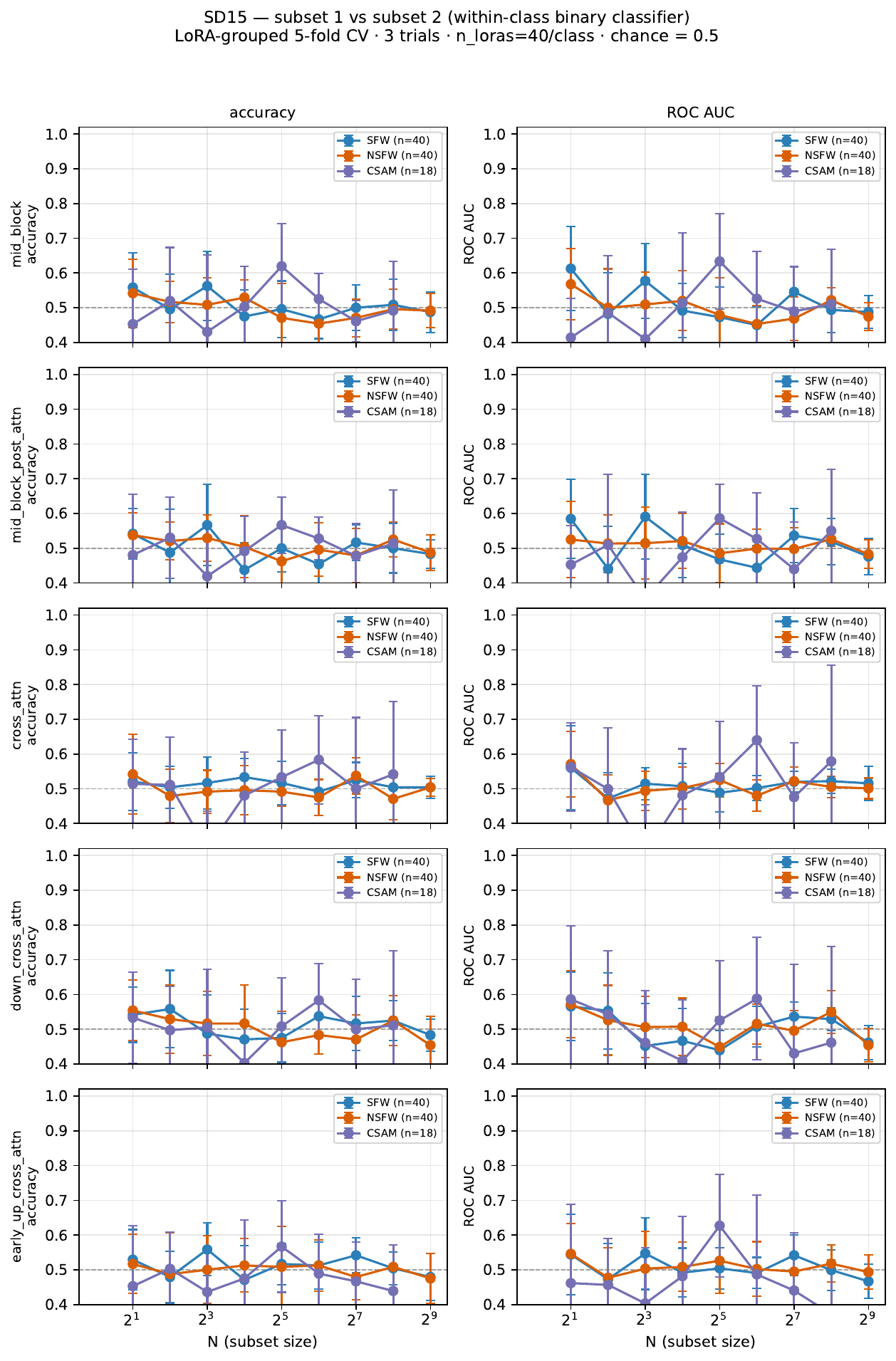}
    \caption{Linear classifier trained on random subsets of varying sizes from the total pool of probes for a sample of LoRAs from each class for SD1.5. Across all layers we extract the classifier is unable to distinguish the averaged probe across any subset size. Thus indicating that the same initialization does not need to be used across all LoRAs or classes.}
    \label{fig:subset_dist_sd15}
\end{figure}

\begin{figure}[h]
    \centering
    \includegraphics[width=0.8\linewidth]{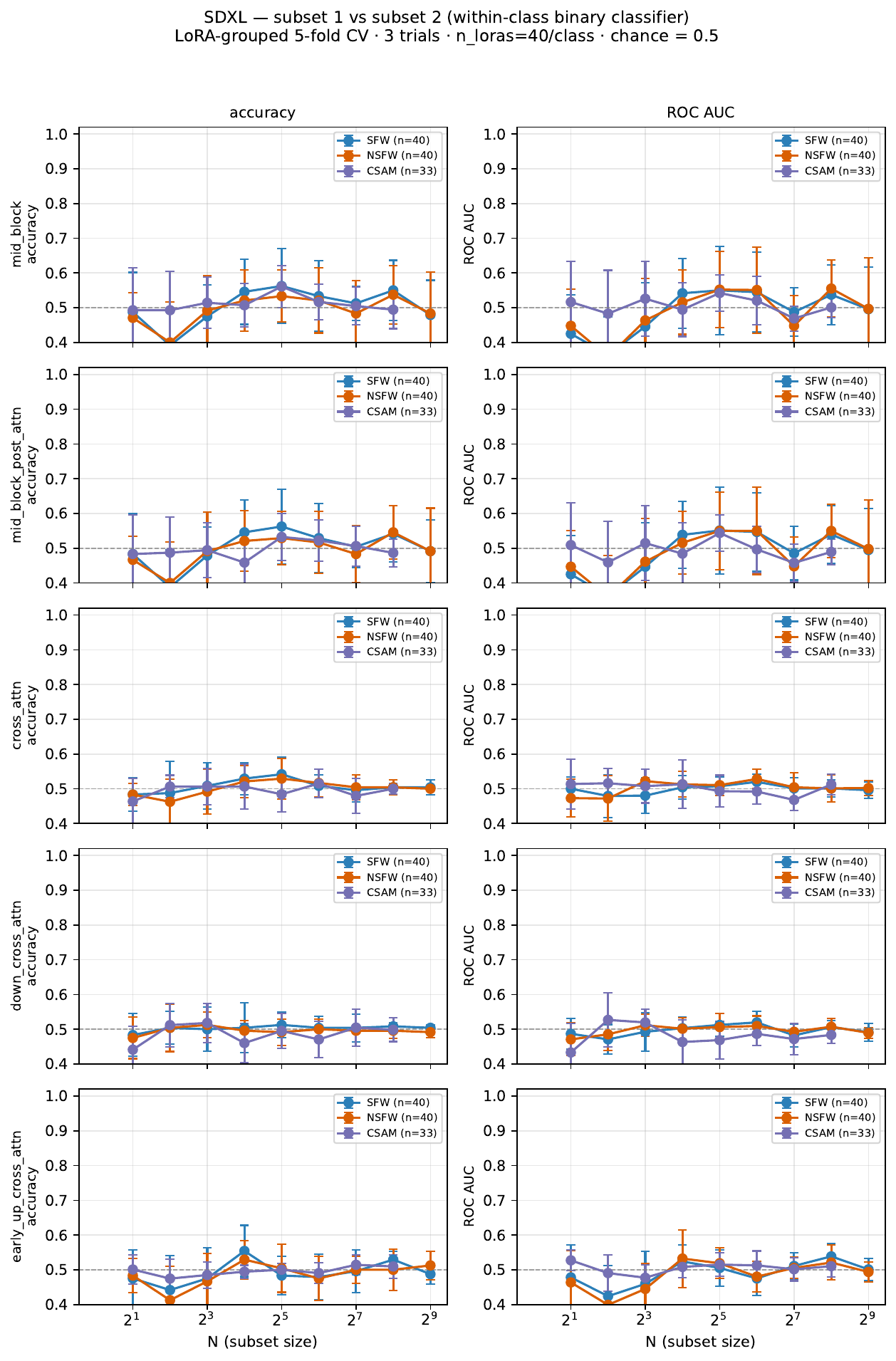}
    \caption{Linear classifier trained on random subsets of varying sizes from the total pool of probes for a sample of LoRAs from each class for SDXL 1.0. Across all layers we extract the classifier is unable to distinguish the averaged probe across any subset size. Thus indicating that the same initialization does not need to be used across all LoRAs or classes.}
    \label{fig:subset_dist_sdxl}
\end{figure}

\begin{figure}[h]
    \centering
    \includegraphics[width=0.8\linewidth]{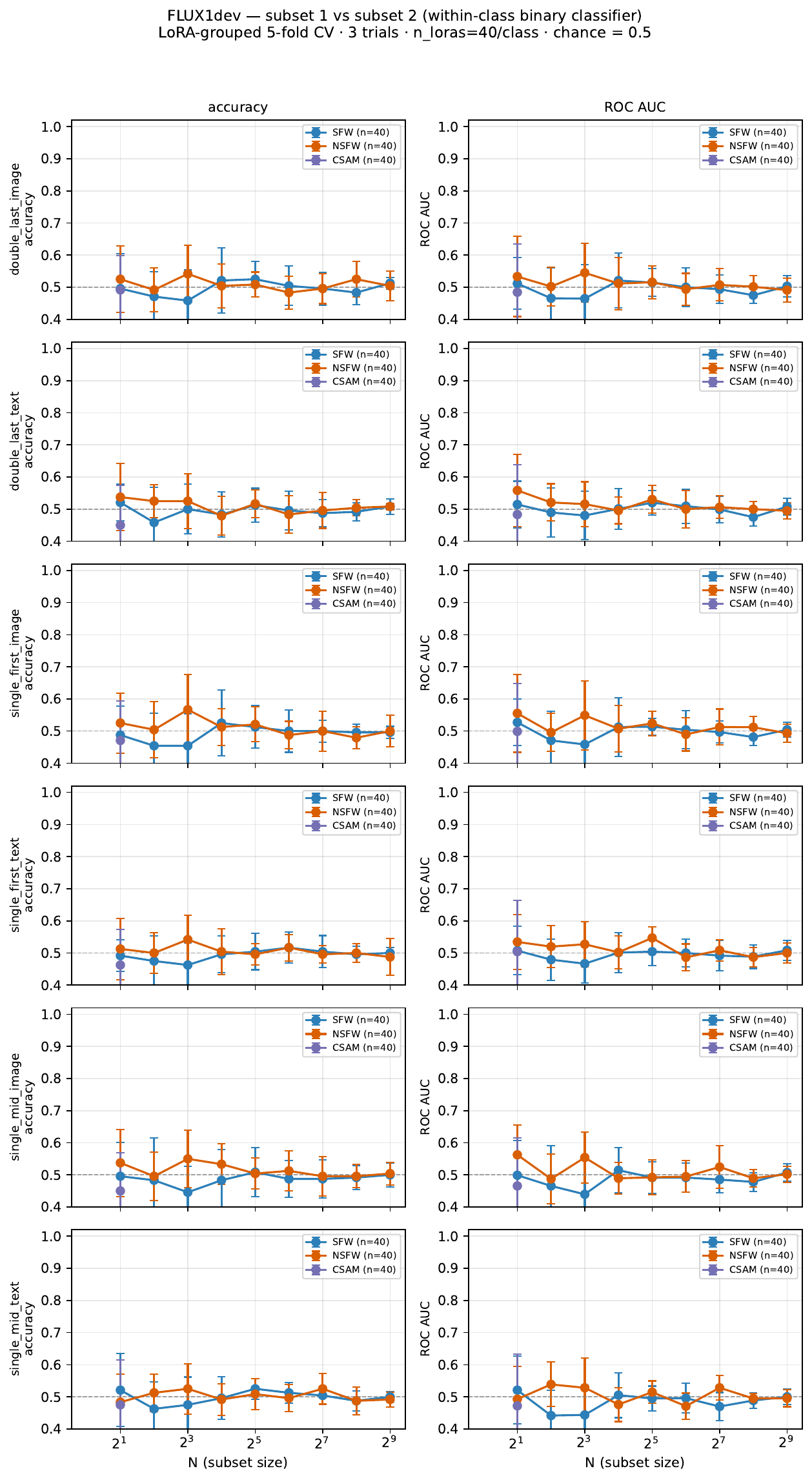}
    \caption{Linear classifier trained on random subsets of varying sizes from the total pool of probes for a sample of LoRAs from each class for FLUX.1-dev. Across all layers we extract the classifier is unable to distinguish the averaged probe across any subset size. Thus indicating that the same initialization does not need to be used across all LoRAs or classes.}
    \label{fig:subset_dist_flux1dev}
\end{figure}

\clearpage

\section{Structured Study}
\label{app:structured}

We present additional results from our structured study of Gaussian probing. Primarily, the standard cross validation results, additional metrics for the leave dataset out results, and ablations on both the impact of using probes from different modules of the model, and ablations on the ensembling strategy for these different modules. 

\subsection{LoRA Training Details}

For Section~\ref{Sec:Explore} when creating the controlled study we trained our own LoRAs as described in the experimental section. For LoRA training we randomly picked the hyperparameters from the following lists: rank = $\{4, 8, 16, 32, 64, 128, 256\}$, learning rate = $\{1E-3, 1E-4, 1E-5\}$, module = $\{\text{text encoder}, \text{latent diffusion model}, \text{both}\}$, and alpha = $\{4, 8, 16, 32, 64, 128, 256\}$, number of finetuning steps = $\{1, 10, 100, 1000\}$, number of examples in finetuning dataset = $\{1, 5, 10, 25, 100, 250, 500, 1000\}$. We randomly assign a configuration of hyperparameters to each LoRA for training to prevent any confounds being exploited during prediction. We use both the diffusers library~\cite{von-platen-etal-2022-diffusers} and the kohya library~\cite{kohya_ss} to capture the common training setups in the wild. 

\subsection{Standard Cross Validation Results}
\label{app:structured_cv}

In this section we present the standard cross validation results for the controlled study. As mentioned in Section~\ref{Sec:Explore} we view these results as the ceiling of the possible performance since we eliminate many of the confounds that temper performance in the wild. These include differences in training hyperparameters between concepts, LoRAs being produced from different base checkpoints, different training libraries being used, a long tail of different ranks and layer configurations that are updated. 

\subsubsection{Stable Diffusion 1.5}

We present the standard CV results across all metrics over 5 folds in Figures \ref{fig:sd15_standard_cv} and \ref{fig:sd15_standard_cv_bars} for SD 1.5.

\begin{figure}[h]
    \centering
    
    \begin{subfigure}[t]{0.45\textwidth}
        \centering
        \includegraphics[width=\linewidth]{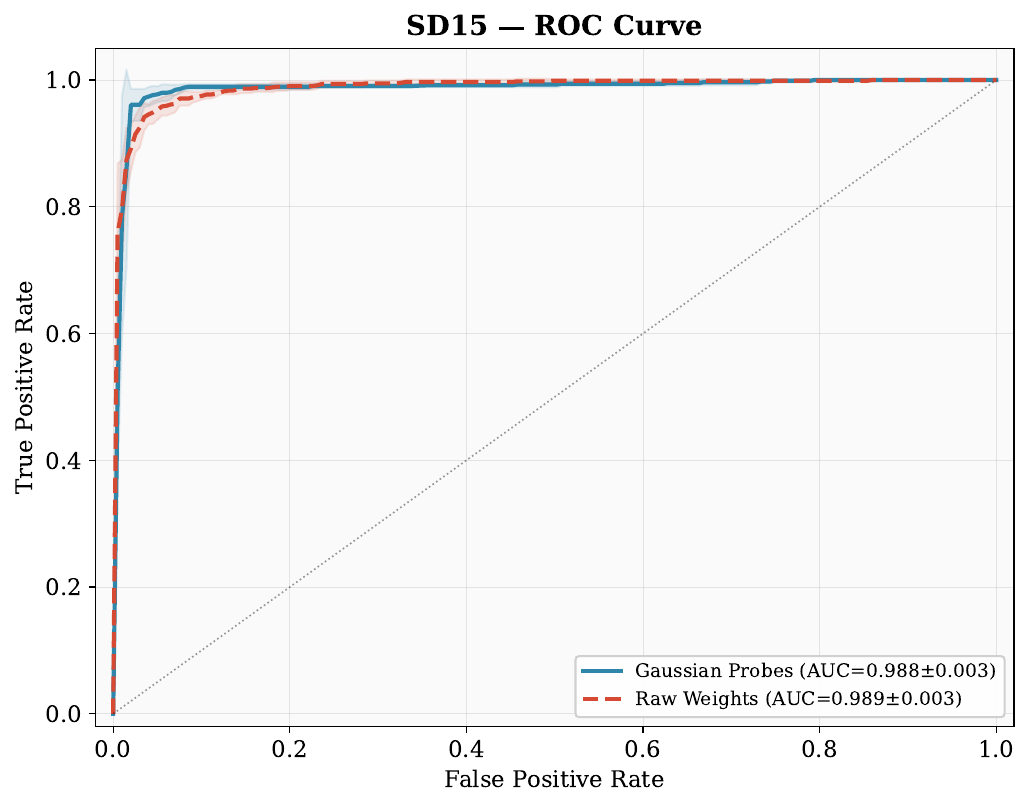}
        \caption{ROC curve of both weight projection and gaussian probing methods in identifying SFW vs NSFW LoRAs for SD 1.5. Both methods perform well showing almost perfect classification. We attribute this to controlling many confounds that appear in the wild, making the task much easier.}
        \label{fig:panel_a}
    \end{subfigure}
    \hfill
    \begin{subfigure}[t]{0.45\textwidth}
        \centering
        \includegraphics[width=\linewidth]{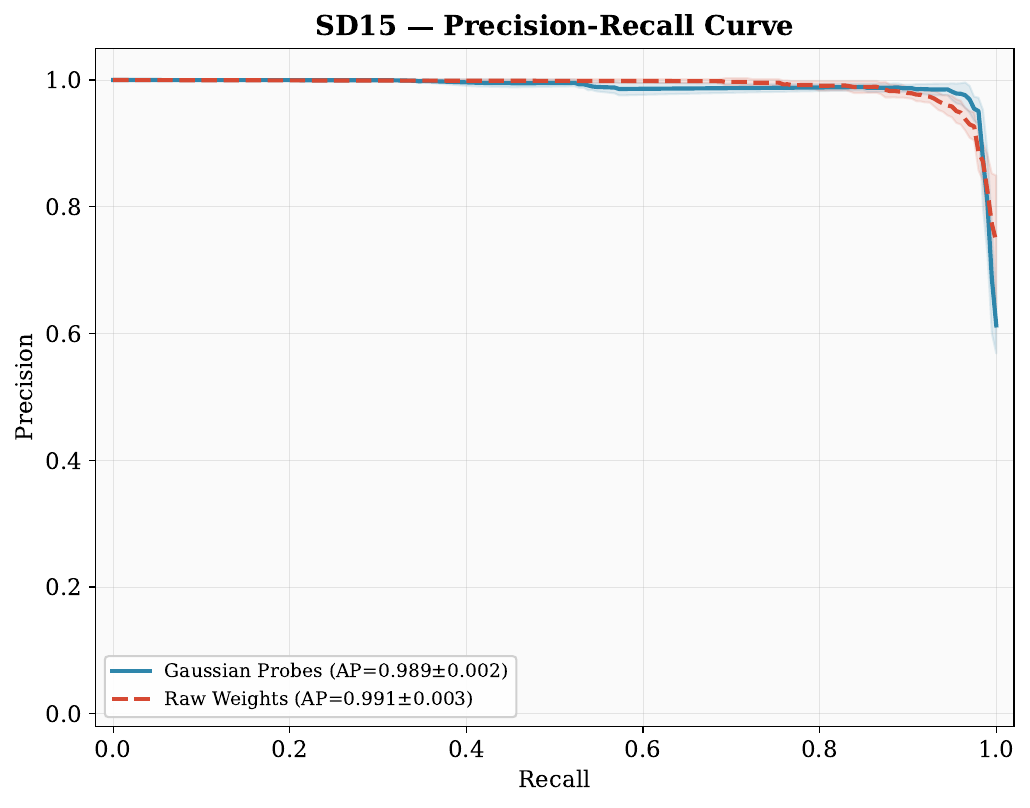}
        \caption{PR curve of both weight projection and gaussian probing methods in identifying SFW vs NSFW LoRAs for SD 1.5. Both methods perform well showing almost perfect classification. We attribute this to controlling many confounds that appear in the wild, making the task much easier.}
        \label{fig:panel_b}
    \end{subfigure}
    \caption{
    Standard cross-validation results for SD 1.5 on controlled adult sexual content prediction task. 
    }
    \label{fig:sd15_standard_cv}
\end{figure}

\begin{figure}[h]
    \centering
    \includegraphics[width=\linewidth]{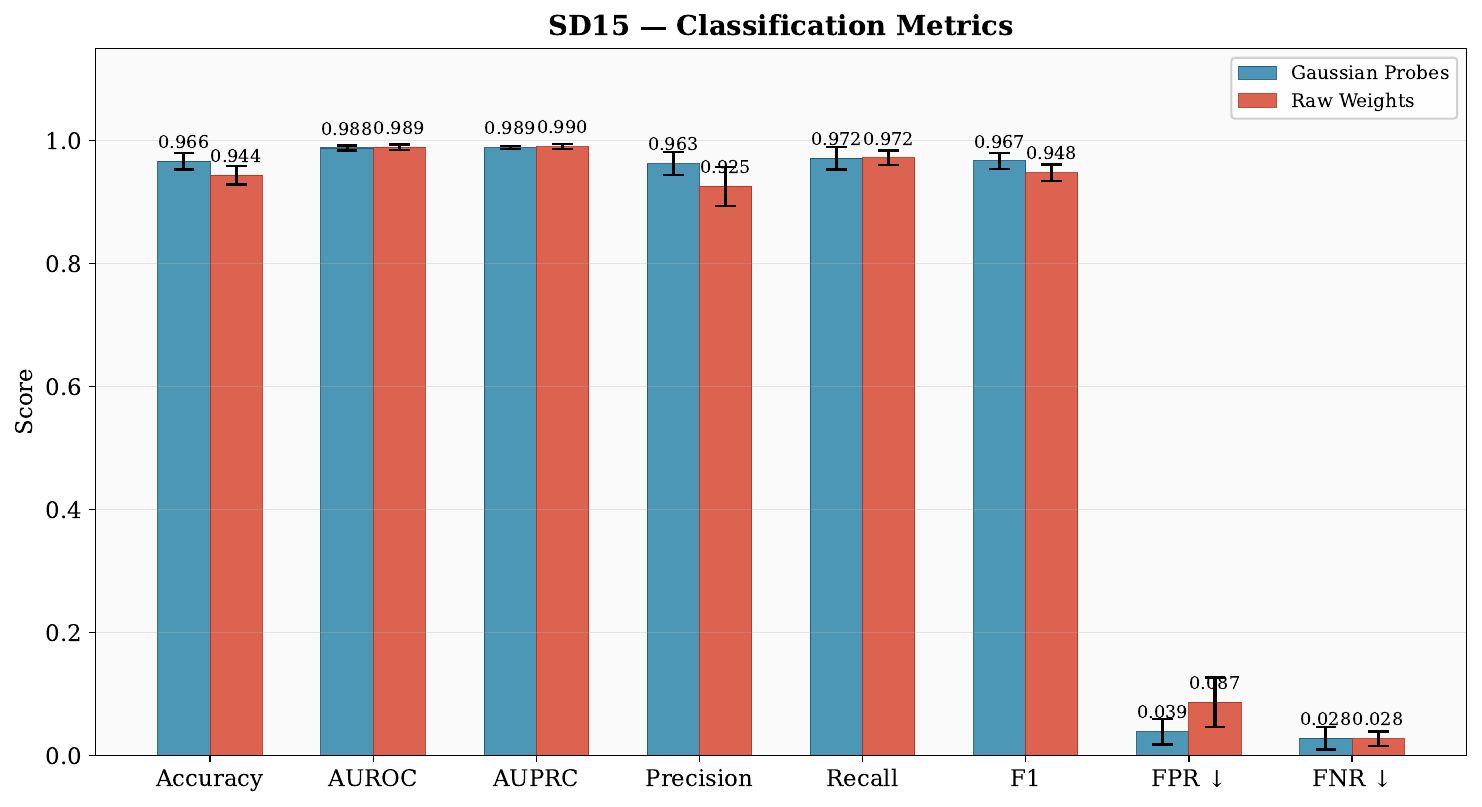}
    \caption{Additional metrics for standard cross-validation results for SD 1.5 on controlled adult sexual content prediction task. }
    \label{fig:sd15_standard_cv_bars}
\end{figure}

\subsubsection{Stable Diffusion XL 1.0}

We present the standard CV results across all metrics over 5 folds in Figures \ref{fig:sdxl_standard_cv} and \ref{fig:sdxl_standard_cv_bars} for SDXL 1.0.

\begin{figure}[h]
    \centering
    
    \begin{subfigure}[t]{0.45\textwidth}
        \centering
        \includegraphics[width=\linewidth]{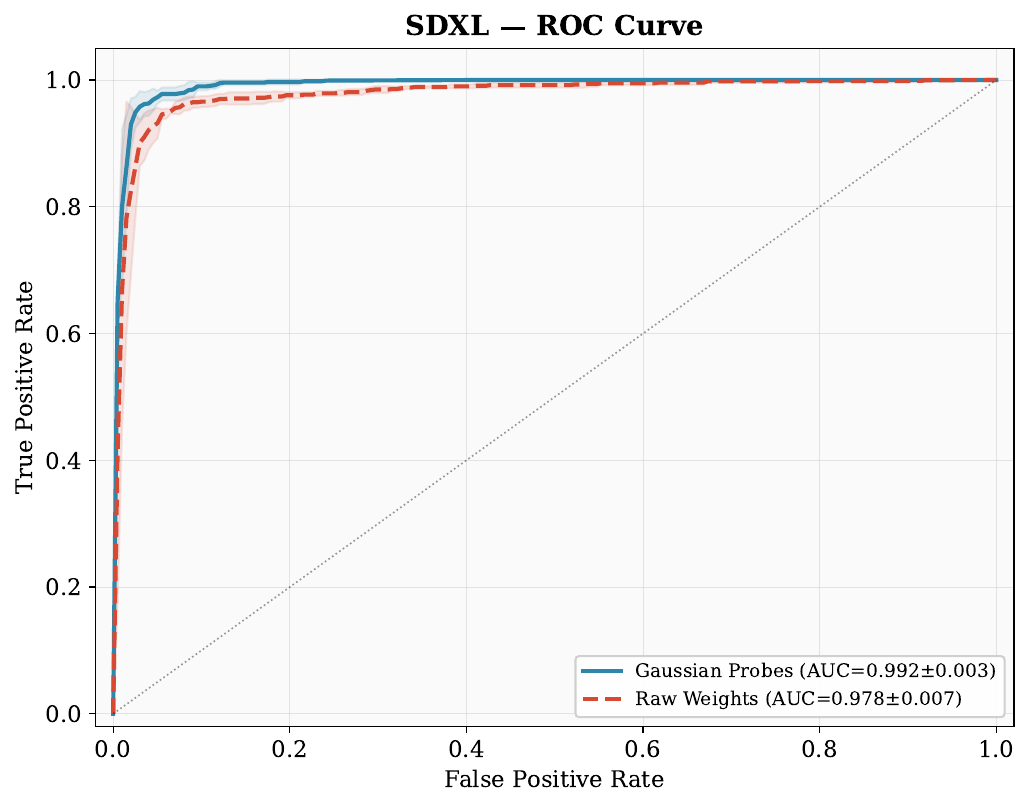}
        \caption{ROC curve of both weight projection and gaussian probing methods in identifying SFW vs NSFW LoRAs for SDXL 1.0. Both methods perform well showing almost perfect classification. We attribute this to controlling many confounds that appear in the wild, making the task much easier.}
        \label{fig:panel_a}
    \end{subfigure}
    \hfill
    \begin{subfigure}[t]{0.45\textwidth}
        \centering
        \includegraphics[width=\linewidth]{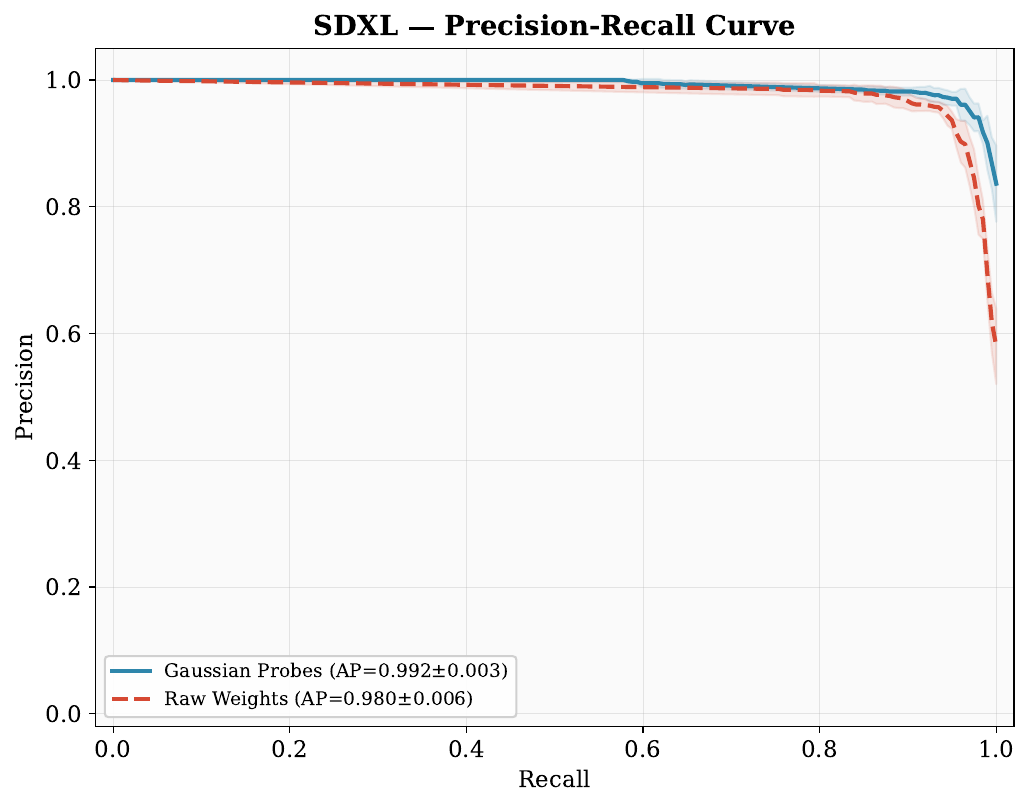}
        \caption{PR curve of both weight projection and gaussian probing methods in identifying SFW vs NSFW LoRAs for SDXL 1.0. Both methods perform well showing almost perfect classification. We attribute this to controlling many confounds that appear in the wild, making the task much easier.}
        \label{fig:panel_b}
    \end{subfigure}
    \caption{
    Standard cross-validation results for SDXL 1.0 on controlled adult sexual content prediction task. 
    }
    \label{fig:sdxl_standard_cv}
\end{figure}

\begin{figure}[h]
    \centering
    \includegraphics[width=\linewidth]{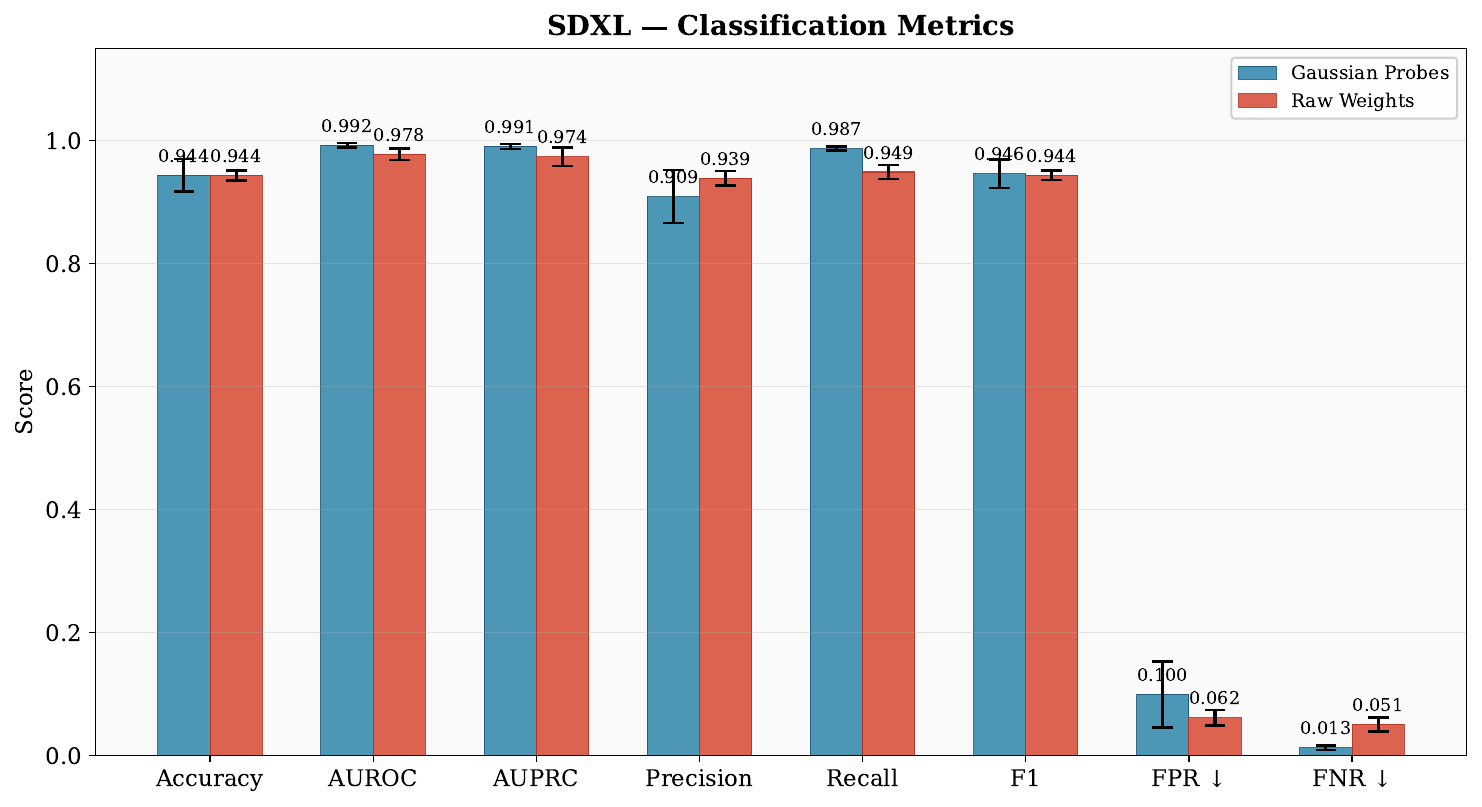}
    \caption{Additonal metrics for standard cross-validation results for SDXL 1.0 on controlled adult sexual content prediction task.  }
    \label{fig:sdxl_standard_cv_bars}
\end{figure}

\subsubsection{FLUX.1-dev}

We present the standard CV results across all metrics over 5 folds in Figures \ref{fig:flux1dev_standard_cv} and \ref{fig:flux1dev_standard_cv_bars} for FLUX.1-dev.

\begin{figure}[h]
    \centering
    
    \begin{subfigure}[t]{0.45\textwidth}
        \centering
        \includegraphics[width=\linewidth]{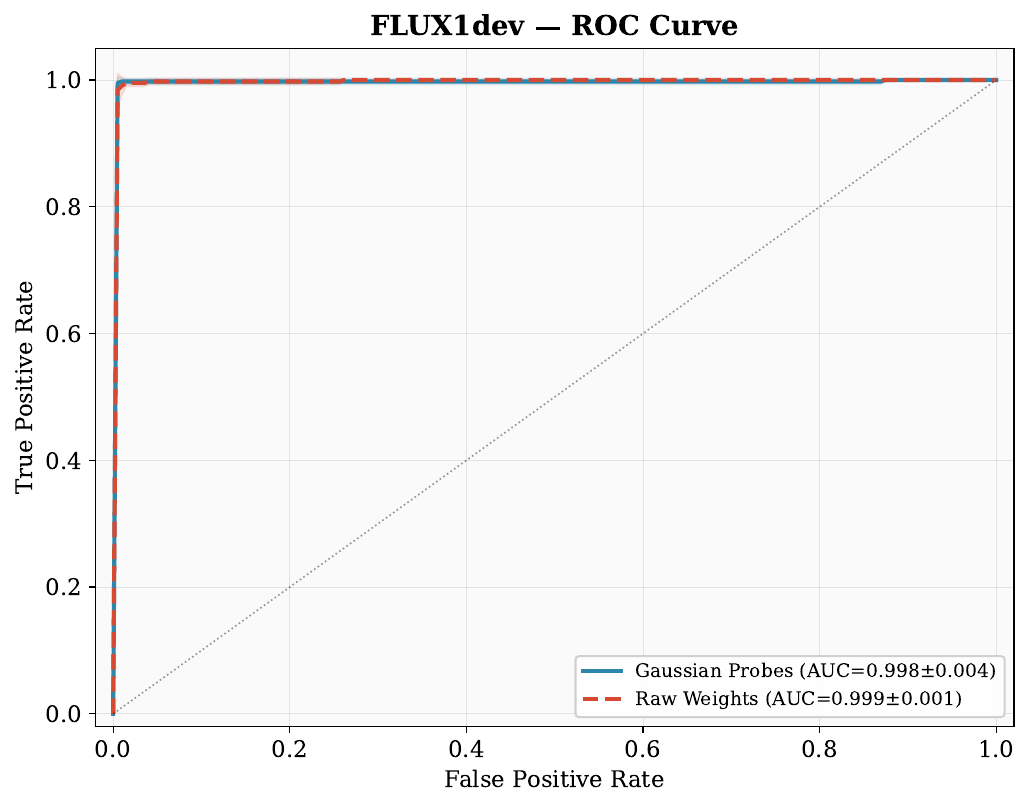}
        \caption{ROC curve of both weight projection and gaussian probing methods in identifying SFW vs NSFW LoRAs for FLUX.1-dev. Both methods perform well showing almost perfect classification. We attribute this to controlling many confounds that appear in the wild, making the task much easier.}
        \label{fig:panel_a}
    \end{subfigure}
    \hfill
    \begin{subfigure}[t]{0.45\textwidth}
        \centering
        \includegraphics[width=\linewidth]{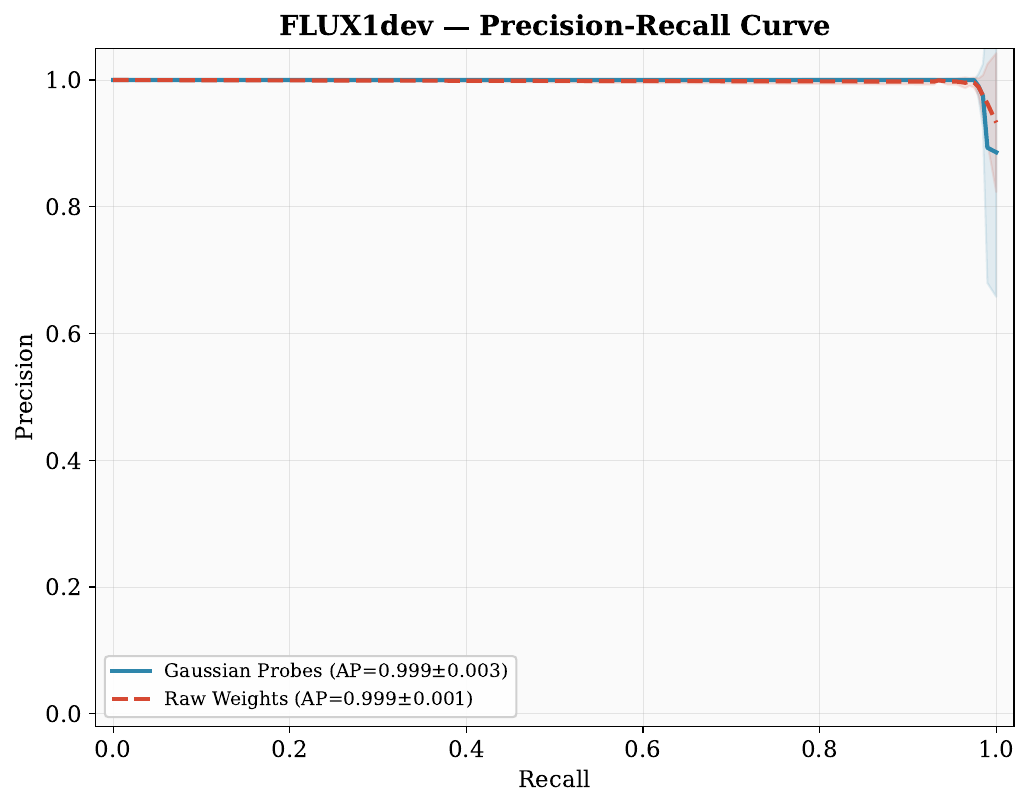}
        \caption{PR curve of both weight projection and gaussian probing methods in identifying SFW vs NSFW LoRAs for FLUX.1-dev. Both methods perform well showing almost perfect classification. We attribute this to controlling many confounds that appear in the wild, making the task much easier.}
        \label{fig:panel_b}
    \end{subfigure}
    \caption{
    Standard cross-validation results for SDXL FLUX.1-dev on controlled adult sexual content prediction task. 
    }
    \label{fig:flux1dev_standard_cv}
\end{figure}

\begin{figure}[h]
    \centering
    \includegraphics[width=\linewidth]{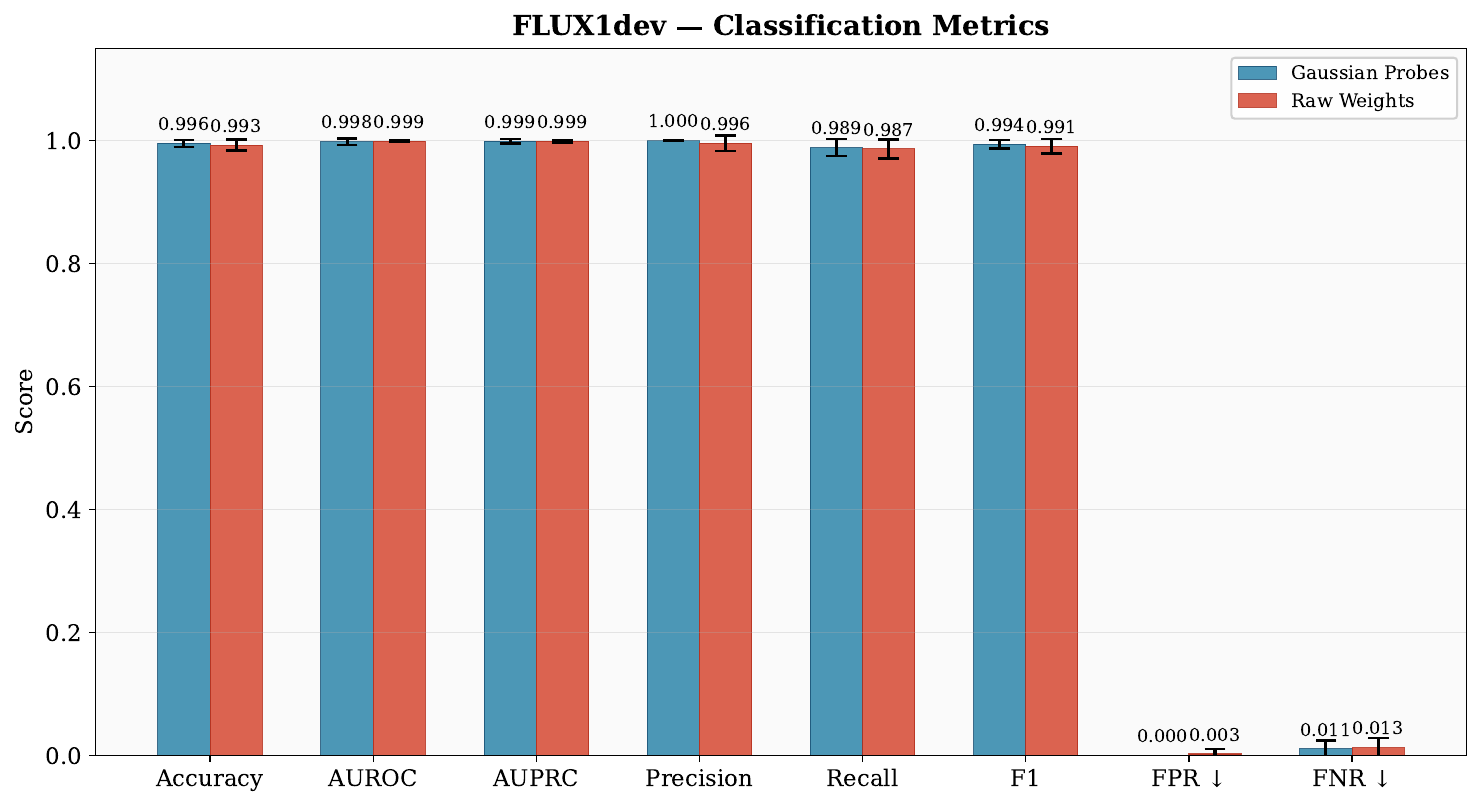}
    \caption{Additional metrics for standard cross-validation results for SDXL FLUX.1-dev on controlled adult sexual content prediction task. }
    \label{fig:flux1dev_standard_cv_bars}
\end{figure}

\subsection{Ensemble and Module Probe Ablation}

In this section, we present the impact of ensembling the representations across the different modules in the model (i.e. the text encoder, the latent diffusion model, or both) in Figure~\ref{fig:module_ablation}. We show that the ensemble is most important for SD 1.5 and there are diminishing returns as the size of the model increases (i.e. as we go from SDXL 1.0 to FLUX.1-dev).

\label{app:controlled_module_ensemble}
\begin{figure}[h]
    \centering
    \includegraphics[width=\linewidth]{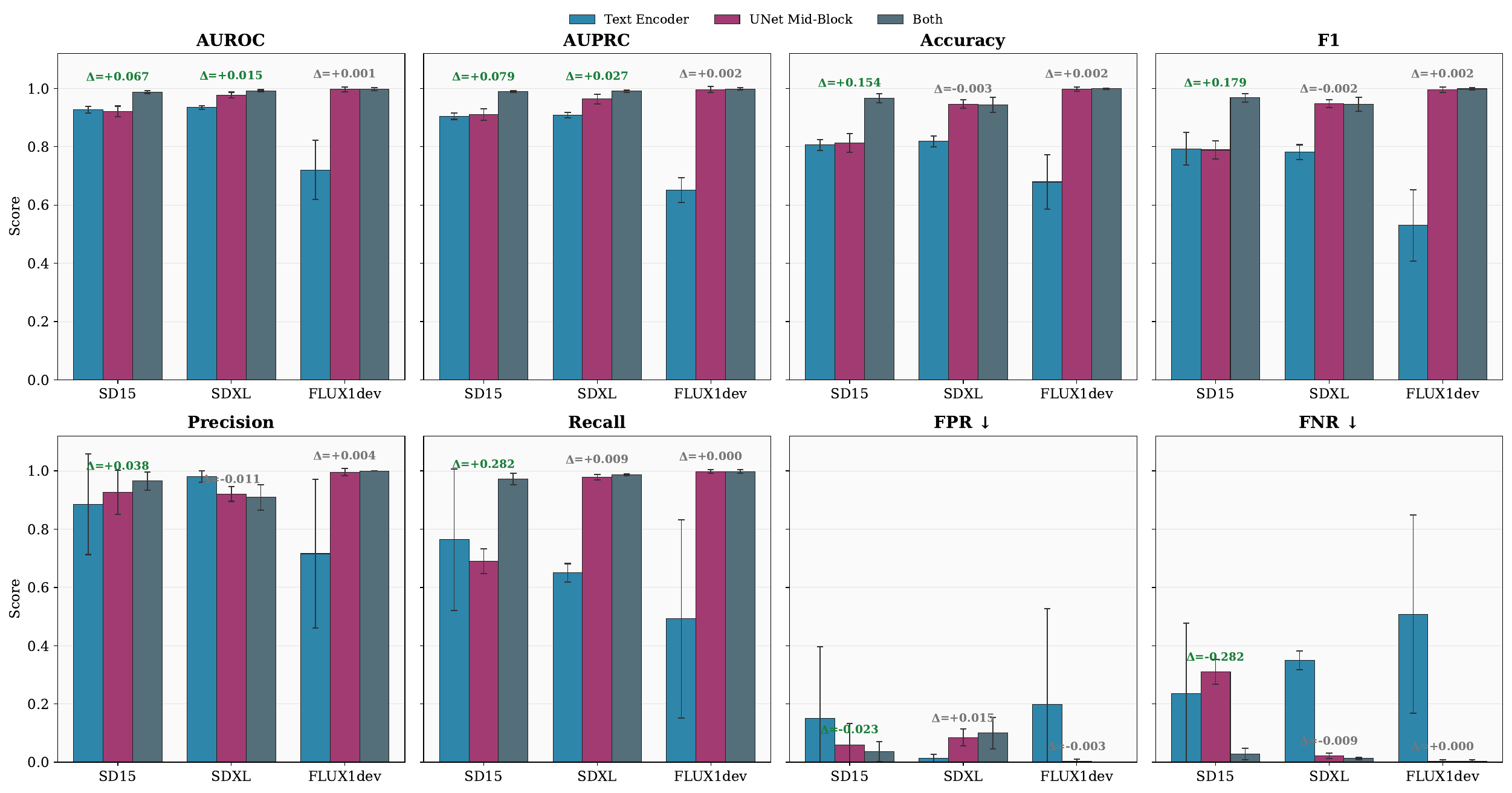}
    \caption{Module ablation for Gaussian probing across architectures. We compare probes extracted from the text encoder alone, the UNet mid-block alone, and both modules combined, across eight metrics. $\Delta$ values show the improvement from ensembling both modules over the best single module. Combining text-encoder and UNet probes is essential for SD 1.5 (AUROC $\Delta$ = +0.067, recall $\Delta$ = +0.282) and yields smaller but consistent gains on SDXL and FLUX.1-dev. The large SD 1.5 recall gain reflects finetunes that modify only the text encoder, UNet probes alone miss these entirely.}
    \label{fig:module_ablation}
\end{figure}

\clearpage
\section{Detecting CSAM Specialization in the Wild Study}

We present additional details from our in the wild study. First, we discuss how we curated our dataset of LoRAs and performed filtering to capture SFW LoRAs from NSFW LoRAs. Second we present additional results in terms of the standard cross validation results and ablations on both the impact of layer choice for SDXL 1.0 and FLUX.1-dev, and ablations on the ensembling strategy for these different layer choices. 

\subsection{Dataset Curation}

We focused our in the wild analysis on CivitAI as our platform of choice given the ease with which we can download LoRAs programatically through the API. There exist two endpoints the civitai.green and civitai.com endpoints. We use these two endpoints to help us curate the SFW LoRAs vs. the NSFW LoRAs. The civitai.green endpoint is a SFW version of civitai.com. Thus, we randomly download 1000 LoRAs from this endpoint and then we use an NSFW classifier on the gallery images to determine if it is still NSFW specialized. We also look for keywords in the model name or description that might indicate it was specialized for NSFW content. For curating the NSFW LoRAs we develop a set of keywords related to adult sexual content and use these to find LoRAs which match these search terms through the civitai.com endpoint. We then run an NSFW classifier over the gallery images to ensure at least one image is NSFW since this acts as a good proxy for the intent of the LoRA training.

\subsection{Additional Metrics}

In this section, we present the per-class values across all of our metrics for CSAM detection across all of our architectures in \ref{fig:per_class_csam_metrics} averaged over 5-fold CV.

\begin{figure}[h]
    \centering
    \includegraphics[width=0.8\linewidth]{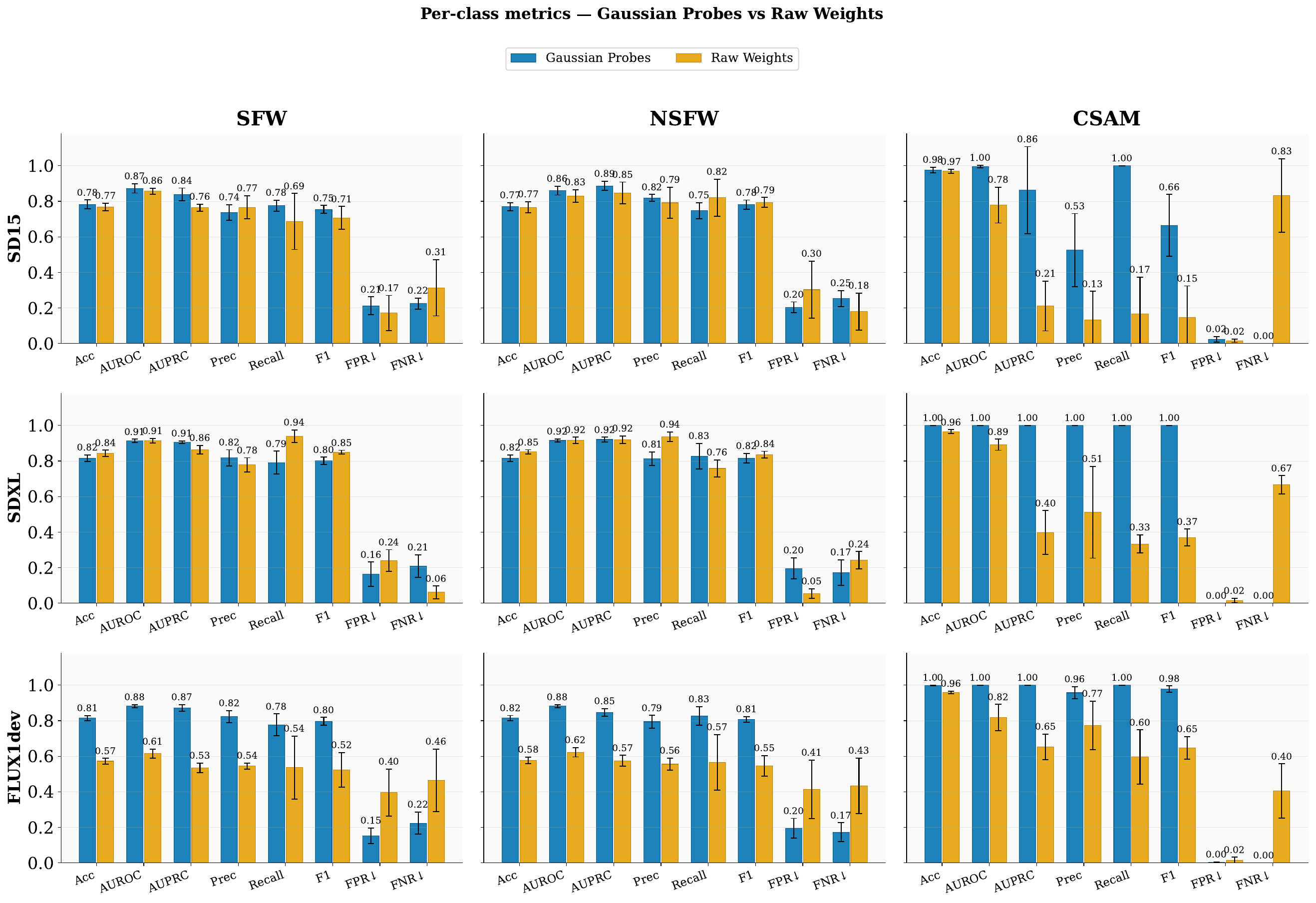}
    \caption{Per-class metrics across architectures (rows: SD 1.5, SDXL, FLUX.1-dev; columns: SFW, NSFW, CSAM). For each (architecture, class) pair we report all eight metrics for Gaussian probes and raw weights. Gaussian probing matches or exceeds raw weights on SFW and NSFW across all metrics, and decisively outperforms on CSAM. Raw weights collapse on the CSAM column due to the small CSAM sample size and high dimensionality.}
    \label{fig:per_class_csam_metrics}
\end{figure}

\clearpage
\subsection{Layer Choice Ablation}

In this section we present our ablations over the individual classifiers trained on the probes from the different layers we picked for SD 1.5, SDXL 1.0 and FLUX.1-dev: Figures~\ref{fig:layer_ablation_sd15}, \ref{fig:layer_ablation_sdxl}, and \ref{fig:layer_ablation_flux}.

\label{app:layer_choice_itw}

\begin{figure}[h]
    \centering
    \includegraphics[width=0.8\linewidth]{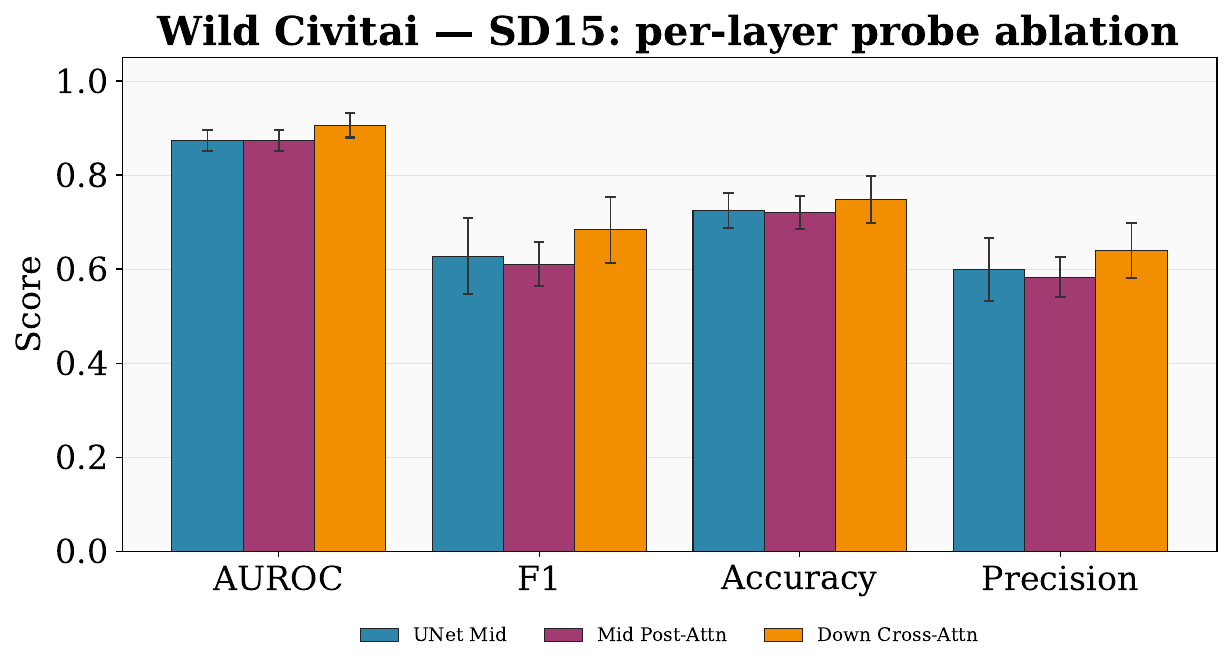}
    \caption{Per-layer probe performance on wild CivitAI SD 1.5 LoRAs. We compare Gaussian probes extracted from three candidate layers: UNet mid-block, mid post-attention, and down cross-attention. No single layer dominates across all metrics, motivating the ensemble approach.}
    \label{fig:layer_ablation_sd15}
\end{figure}

\begin{figure}[h]
    \centering
    \includegraphics[width=0.8\linewidth]{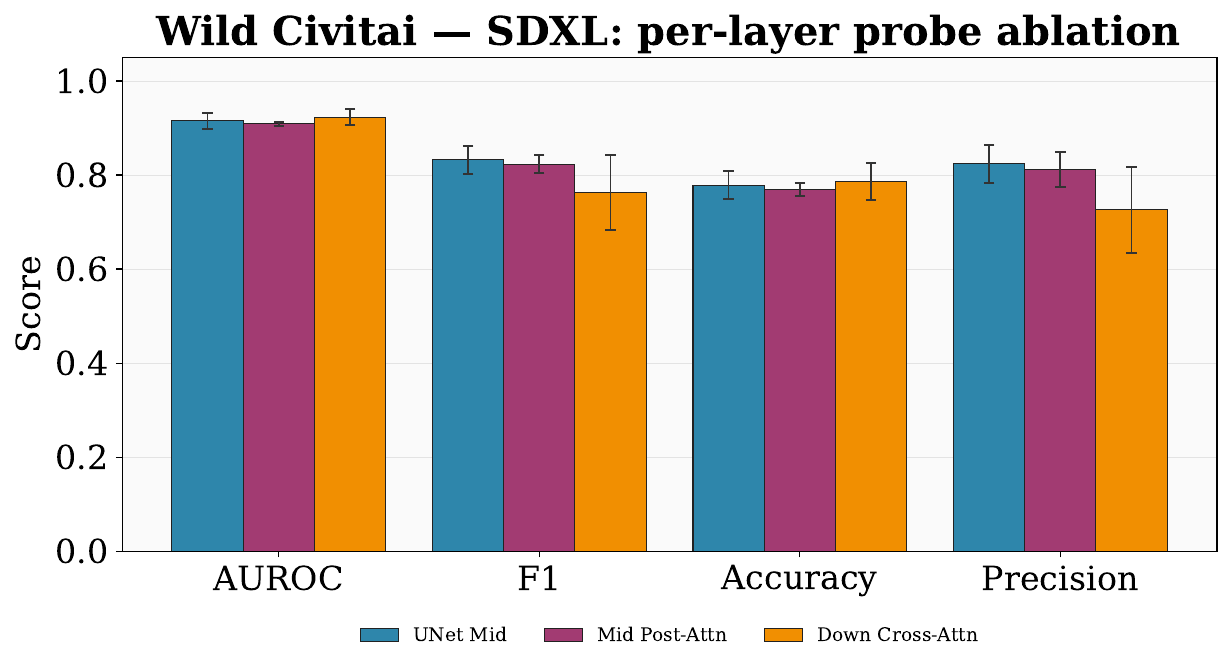}
    \caption{Per-layer probe performance on wild CivitAI SDXL 1.0 LoRAs. We compare Gaussian probes extracted from three candidate layers: UNet mid-block, mid post-attention, and down cross-attention. No single layer dominates across all metrics, motivating the ensemble approach.}
    \label{fig:layer_ablation_sdxl}
\end{figure}

\begin{figure}[h]
    \centering
    \includegraphics[width=0.8\linewidth]{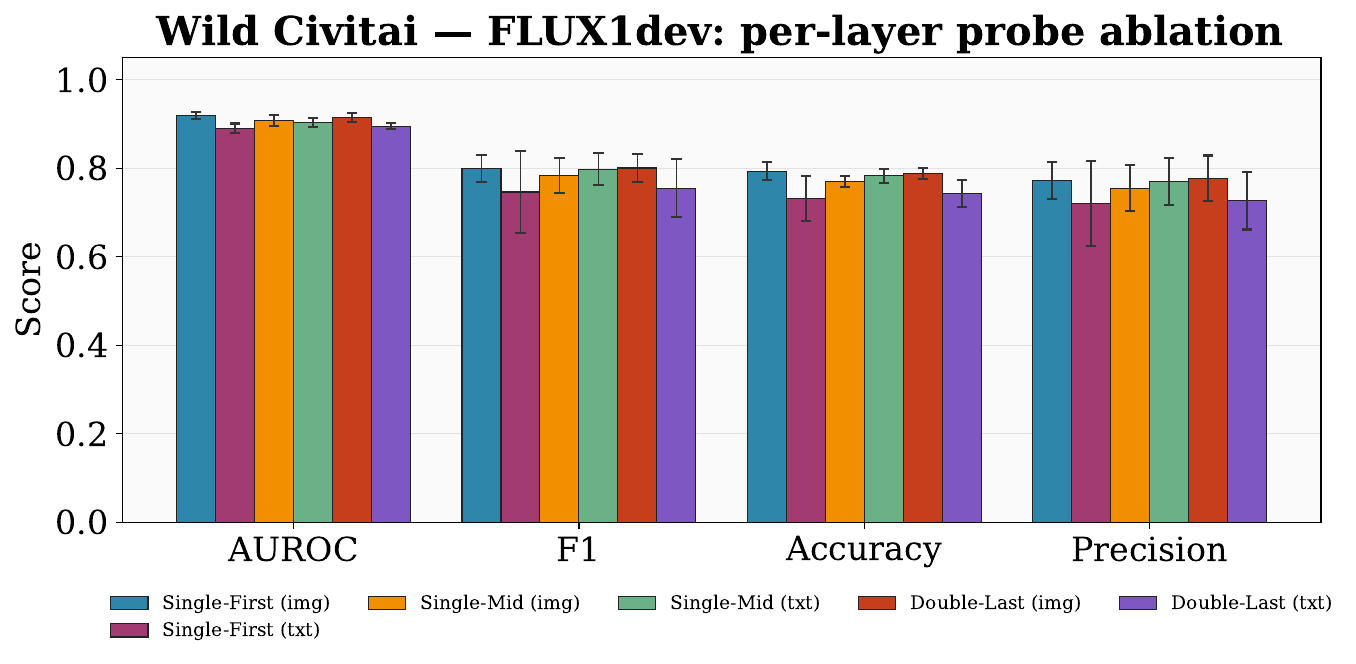}
    \caption{Per-layer probe performance on wild CivitAI FLUX.1-dev LoRAs. We compare Gaussian probes extracted from six candidate layers spanning the first, middle, and last thirds of the transformer, for both image and text streams. Performance is similar across layer choices, with no single layer clearly outperforming.}
    \label{fig:layer_ablation_flux}
\end{figure}

\clearpage
\subsection{Error Analysis}

In this section we provide the raw counts for the error analysis we present in Section~\ref{Sec:Wild} to complement the error rates and give a sense of the scale of the errors.

\begin{figure}[h]
    \centering
    \includegraphics[width=\linewidth]{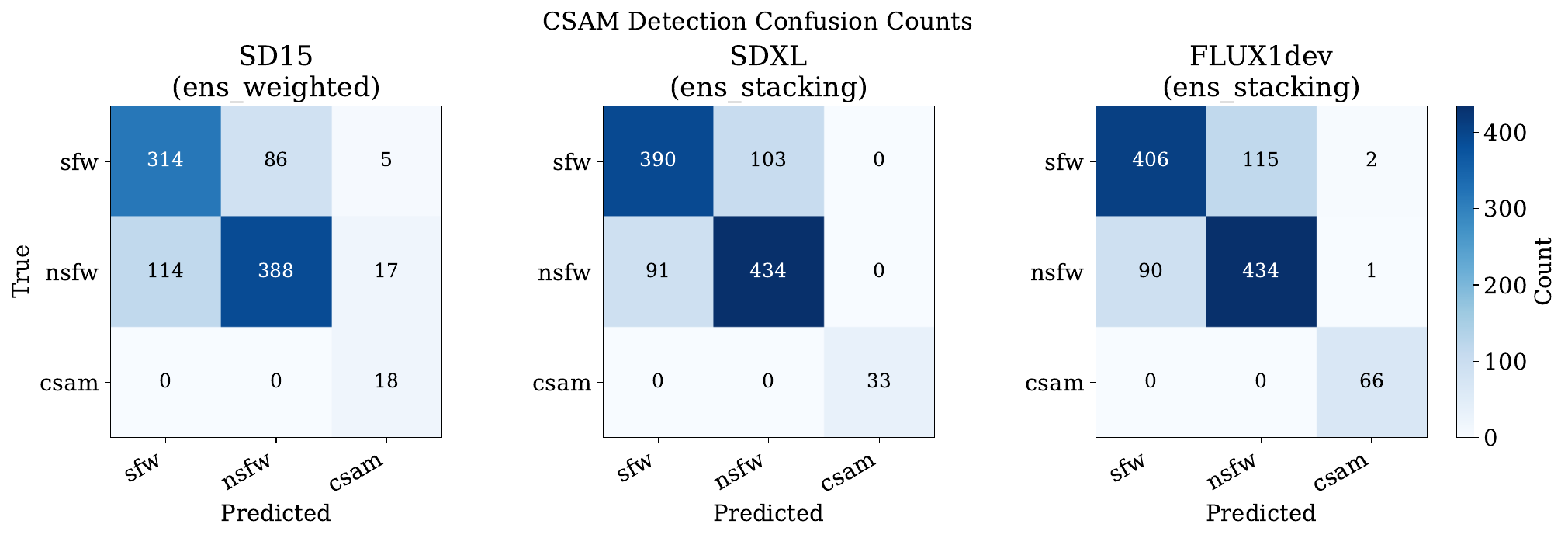}
    \caption{CSAM detection error counts by true class, across architectures. Gaussian probing achieves 100\% recall on CSAM across SD 1.5, SDXL, and FLUX.1-dev (bottom row).}
    \label{fig:placeholder}
\end{figure}

\subsection{Ensembling Ablation}
\label{app:ensembling_choice_itw}

In this section we compare the best individual layer classifier to the ensembled classifier for SD 1.5, SDXL 1.0 and FLUX.1-dev (Figures~\ref{fig:layer_ablation_best_sd15}, \ref{fig:layer_ablation_best_sdxl}, and \ref{fig:layer_ablation_best_flux1dev}). We also explore the best ensembling strategy for each architecture ((Figures~\ref{fig:ensemble_ablation_best_sd15}, \ref{fig:ensemble_ablation_best_sdxl}, and \ref{fig:ensemble_ablation_best_flux1dev}).

\begin{figure}[h]
    \centering
    \includegraphics[width=\linewidth]{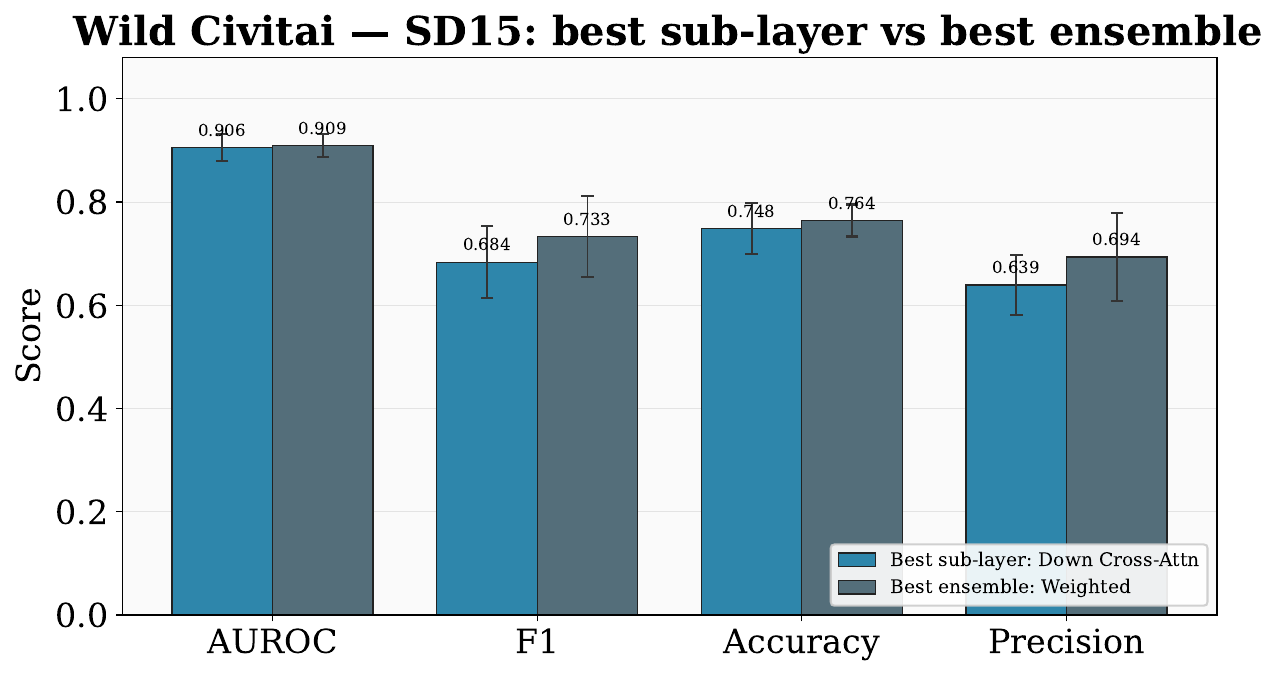}
    \caption{Best single-layer probe vs. best ensemble on wild CivitAI SD15. Ensembling probes across layers (soft vote) improves over the best single layer (down cross-attn) on all metrics.}
    \label{fig:layer_ablation_best_sd15}
\end{figure}

\begin{figure}[h]
    \centering
    \includegraphics[width=\linewidth]{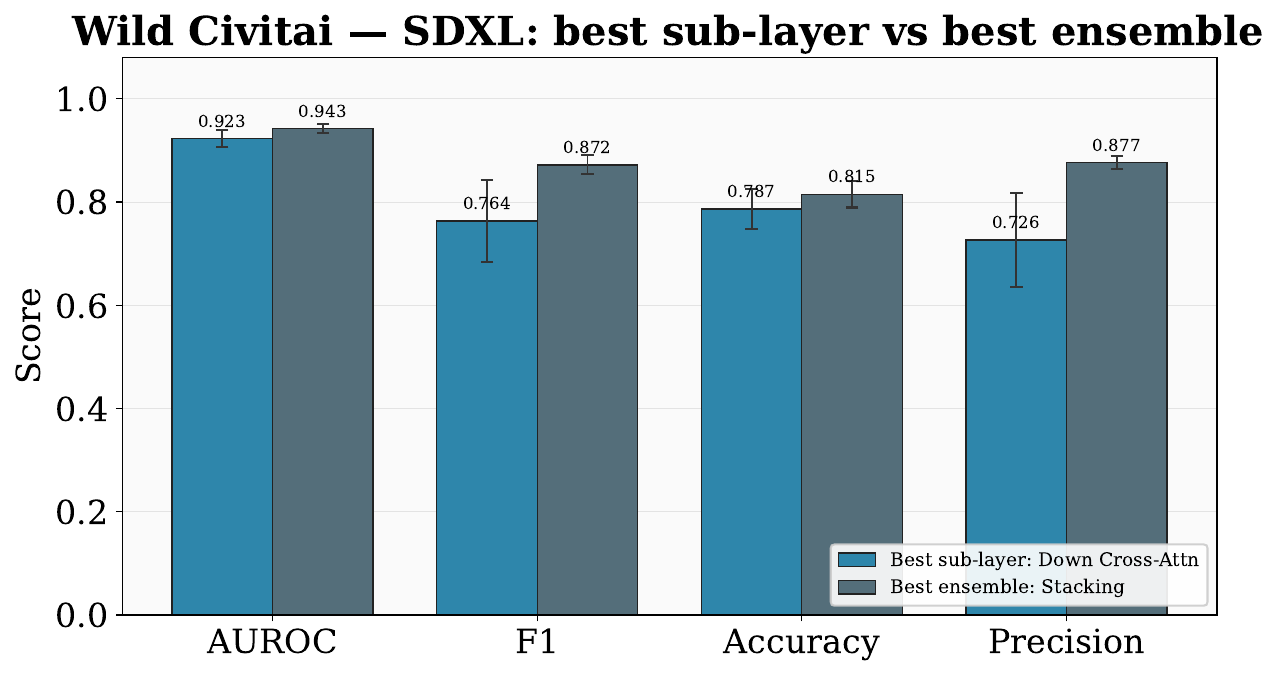}
    \caption{Best single-layer probe vs. best ensemble on wild CivitAI SDXL 1.0. Ensembling probes across layers (stacking) improves over the best single layer (down cross-attn) on all metrics.}
    \label{fig:layer_ablation_best_sdxl}
\end{figure}

\begin{figure}[h]
    \centering
    \includegraphics[width=\linewidth]{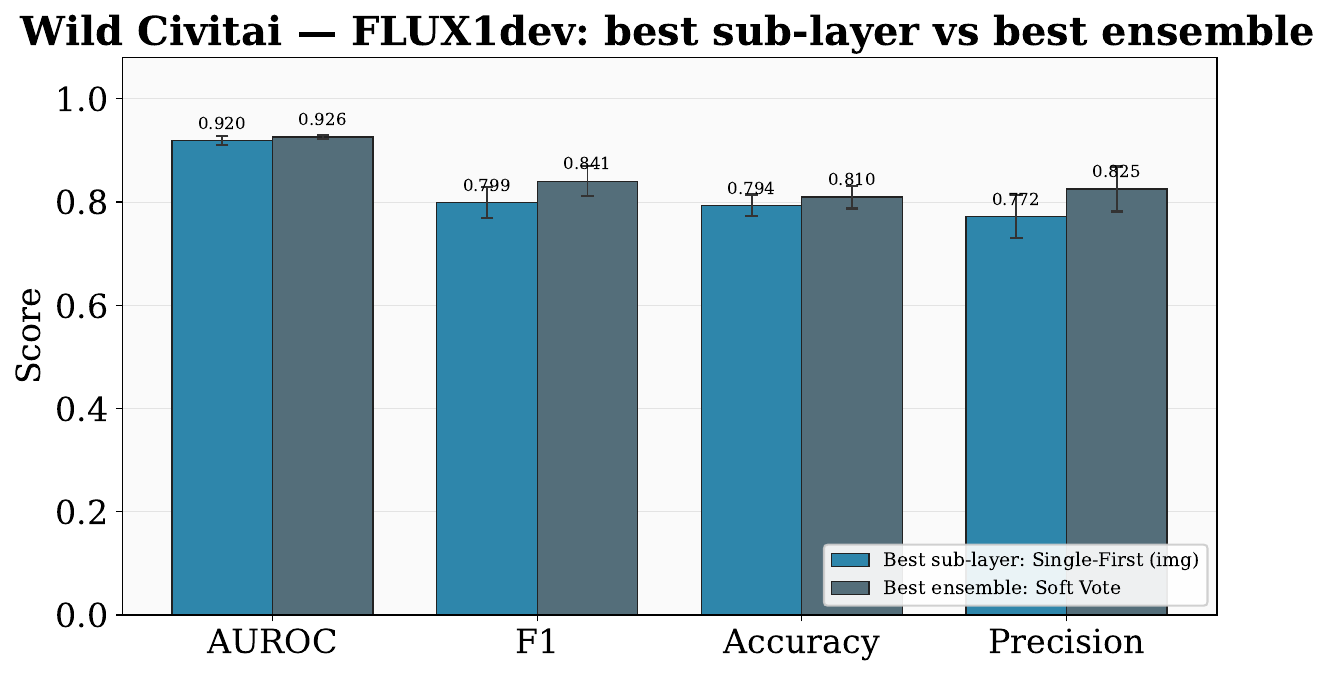}
    \caption{Best single-layer probe vs. best ensemble on wild CivitAI FLUX.1-dev. Ensembling probes across layers (soft vote) improves over the best single layer (single first image) on all metrics.}
    \label{fig:layer_ablation_best_flux1dev}
\end{figure}

\begin{figure}[h]
    \centering
    \includegraphics[width=\linewidth]{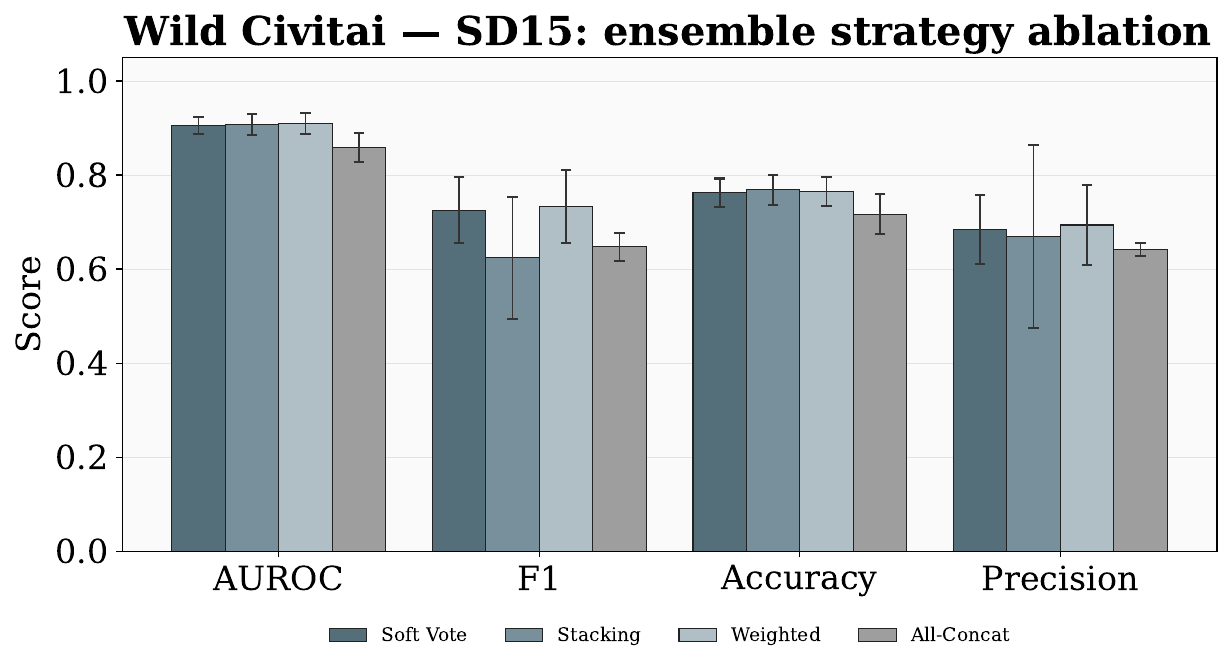}
    \caption{Comparing performance of different ensembling strategies on wild CivitAI SD 1.5. Weighted soft vote seems to provide the best performance across the metrics.}
    \label{fig:ensemble_ablation_best_sd15}
\end{figure}

\begin{figure}[h]
    \centering
    \includegraphics[width=\linewidth]{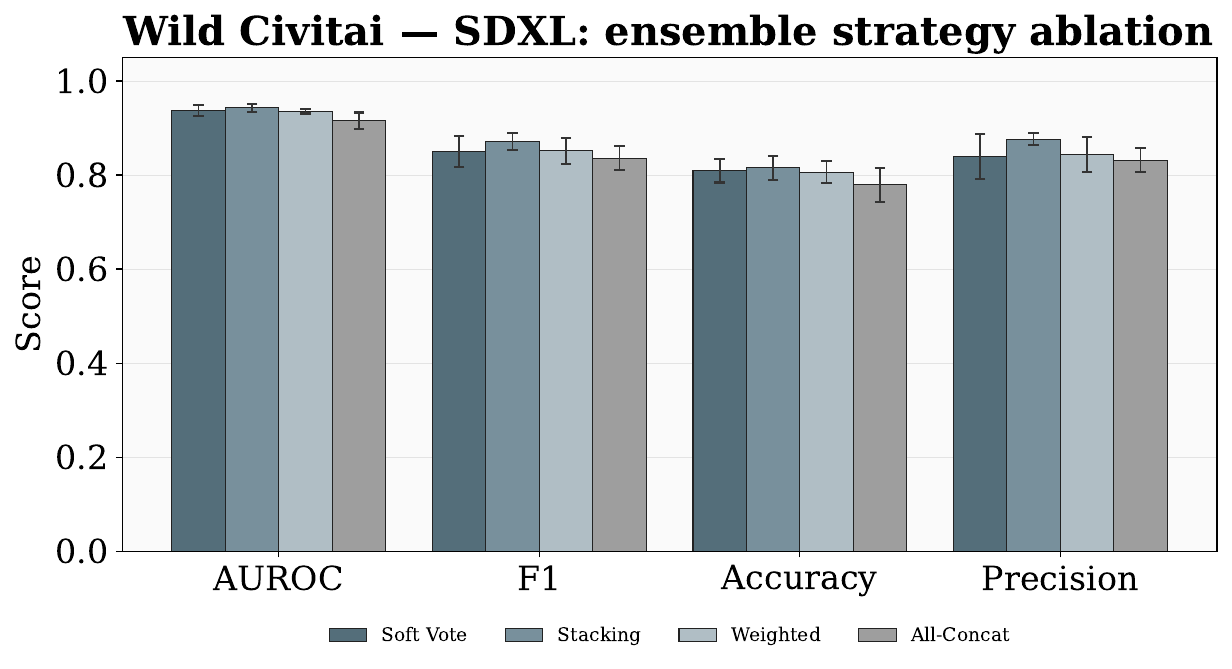}
    \caption{Comparing performance of different ensembling strategies on wild CivitAI SDXL 1.0. Stacking seems to provide the best performance across the metrics.}
    \label{fig:ensemble_ablation_best_sdxl}
\end{figure}

\begin{figure}[h]
    \centering
    \includegraphics[width=\linewidth]{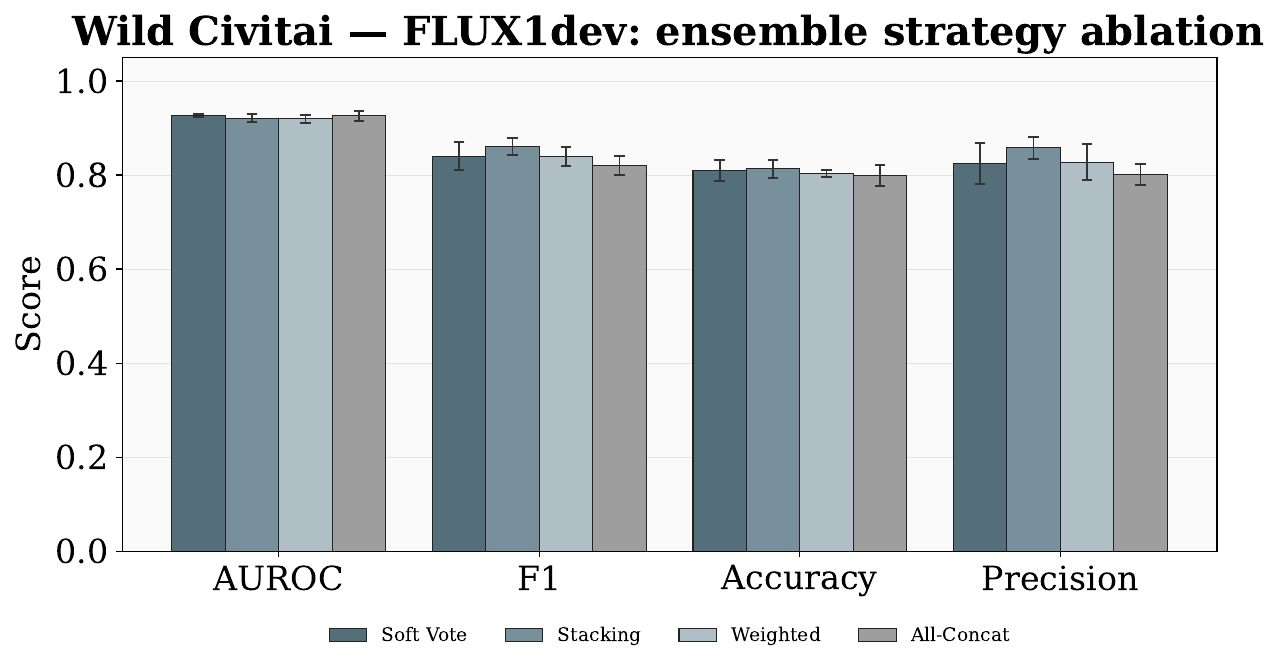}
    \caption{Comparing performance of different ensembling strategies on wild CivitAI FLUX.1-dev. Standard soft vote seems to provide the best performance across the metrics.}
    \label{fig:ensemble_ablation_best_flux1dev}
\end{figure}

\end{document}

%% file: intro_option2.tex
The proliferation of open weight generative models, such as Stable Diffusion~\citep{rombach2022high}, FLUX~\citep{flux2024}, and Wan~\citep{wan2025wan}, has made high-quality image and video generation widely accessible~\citep{yang2023diffusion, fuest2026diffusion}. In tandem, low-rank adaptation (LoRA), a commonly used finetuning algorithm, enables cheap and efficient specialization of generative models at a fraction of the cost of traditional finetuning~\cite{hu2022lora}. Accessibility has also been improved by user-friendly graphical interfaces such as InvokeAI~\citep{invokeai}, which allow amateurs and hobbyists to finetune models for their own creative purposes.
As a result, powerful image and video models can be easily specialized, and those specializations can be shared through multiple different services, including public model-sharing platforms.

This shift has created new governance challenges for open model-sharing platforms. 
%In closed deployment settings, governance is typically organized around a relatively stable model whose behavior is mediated through a provider-controlled interface. Open model ecosystems, however, present a different challenge.
Platforms such as CivitAI and Hugging Face host and distribute base models, fine-tuned variants, and, crucially, lightweight LoRA adaptors that users can combine, circulate, and redeploy widely before any output is ever inspected. This includes models optimized to produce child sexual abuse material (CSAM), with offenders producing bespoke models through fine-tuning that target particular children, victims and survivors of child sexual abuse \cite{thiel2023generative}.  Governance, therefore, becomes a problem of screening reusable model artifacts before they are broadly distributed.

At present, model assessment for first-party model providers is still largely organized around {\bf generative evaluation}: prompting an adapted model, inspecting its outputs, and using those outputs to infer whether the model has acquired a harmful capability~\citep{thornsafetybydesign}. This approach does not scale to model hosting platforms. It depends on prompt coverage, requires iterative red-teaming and review, and becomes increasingly costly as the number of uploaded variants grows. The scale of adaptors to be screened is substantial: CivitAI alone reports hundreds of thousands of new LoRAs trained in a single month, alongside millions of generations\footnote{\url{https://www.runpod.io/case-studies/civitai-runpod-case-study}}. 
%Prior work has identified nearly 35,000 publicly downloadable harmful deepfake model variants across CivitAI and Hugging Face, the majority hosted on CivitAI~\cite{hawkins2025deepfakes}, while estimates suggest that 16.97\% of CivitAI-hosted models are trained to generate not-safe-for-work (NSFW) content~\citep{wei2024exploring}. 
At this scale, output-based auditing cannot serve as the sole mechanism for pre-distribution governance.

For several categories of harmful content, the limitations are even more acute. Evaluating outputs related to bioweapons, cyberattacks, hate speech, or non-consensual intimate imagery can impose substantial psychological burdens, and the expertise required for reliable evaluation may be scarce~\citep{roberts2019behind, steiger2021psychological, gillespie2018custodians}. For CSAM, our central motivating application, generative evaluation breaks down entirely: attempting to generate CSAM, regardless of intent or success, is unlawful behavior in several regulatory regimes and jurisdictions, including the United States~\citep{shevlane2023modelevaluationextremerisks}. 

This leaves platforms and auditors with limited ability to determine whether a user-uploaded adaptor has specialized a model toward CSAM generation before that adaptor is shared.  LoRAs represent a unique risk, since their small and portable format makes them easy for offenders to exchange \citep{thiel2023generative}. That limitation matters given AI-generated CSAM is a large and accelerating crisis, with cross-sector impact across hotlines, content moderators, law enforcement \citep{thiel2023generative, iwf2026aigcsamupdate}, creating significant human harm\footnote{\url{https://www.cbsnews.com/news/sextortion-generative-ai-scam-elijah-heacock-take-it-down-act/}}. In its most recent report analyzing 2024 reports, NCMEC (the National Center for Missing and Exploited Children, which acts as a global clearinghouse for reports related to child sexual abuse and exploitation) received 67,000 reports of AI-generated CSAM, up from 4,700 in 2023~\citep{ncmec2025cybertipline}.

Taken together, this harm landscape motivates our core questions:
\begin{center}
  \emph{
Can harmful specialization be detected from weights alone,\\ without ever generating an output?}
\end{center}

\begin{figure}[t]
    \centering
    \begin{subfigure}[t]{0.48\textwidth}
        \centering
        \includegraphics[width=\linewidth]{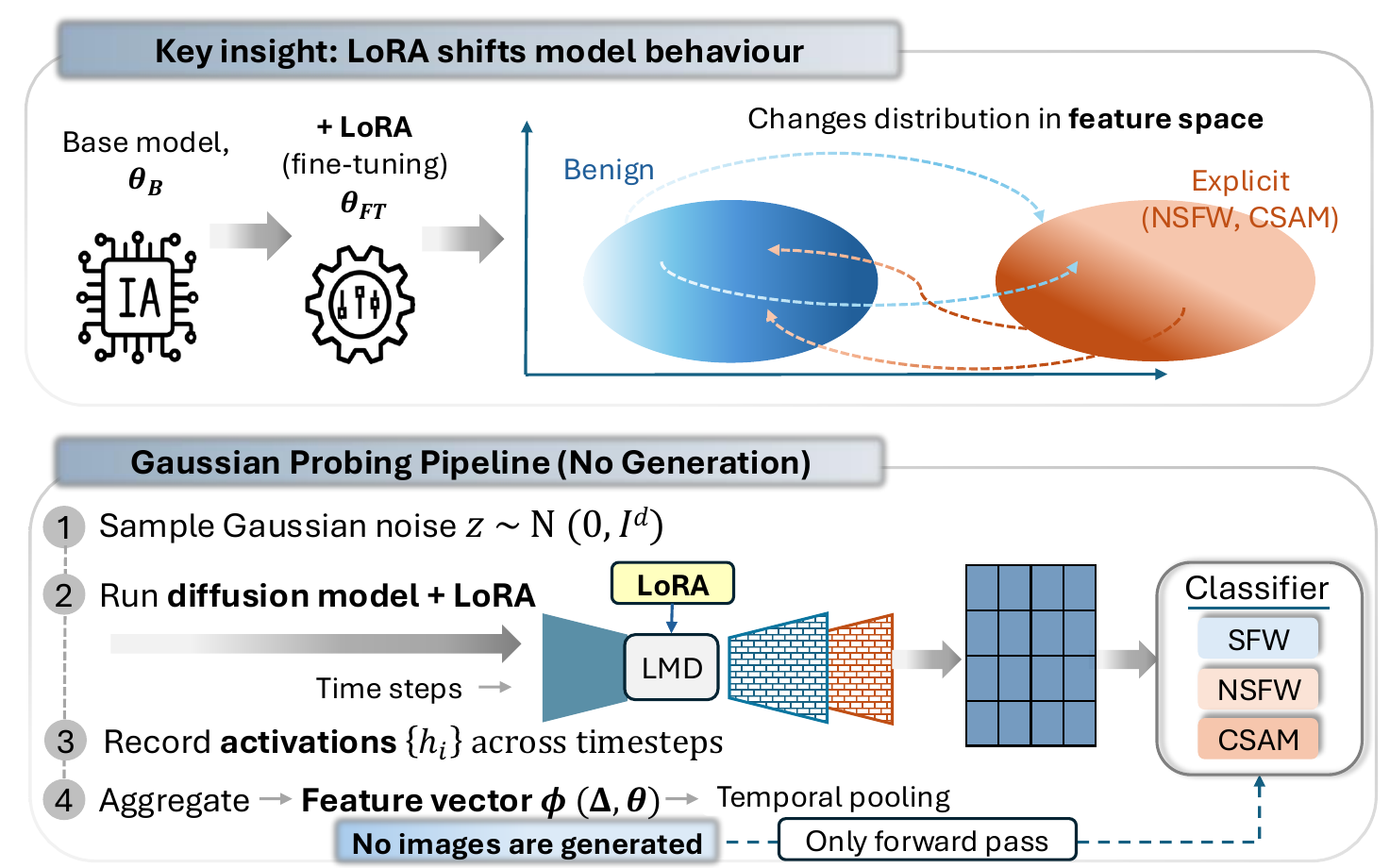}
        \caption{}
        \label{fig:panel_a}
    \end{subfigure}
    \hfill
    \begin{subfigure}[t]{0.48\textwidth}
        \centering
        \includegraphics[width=\linewidth]{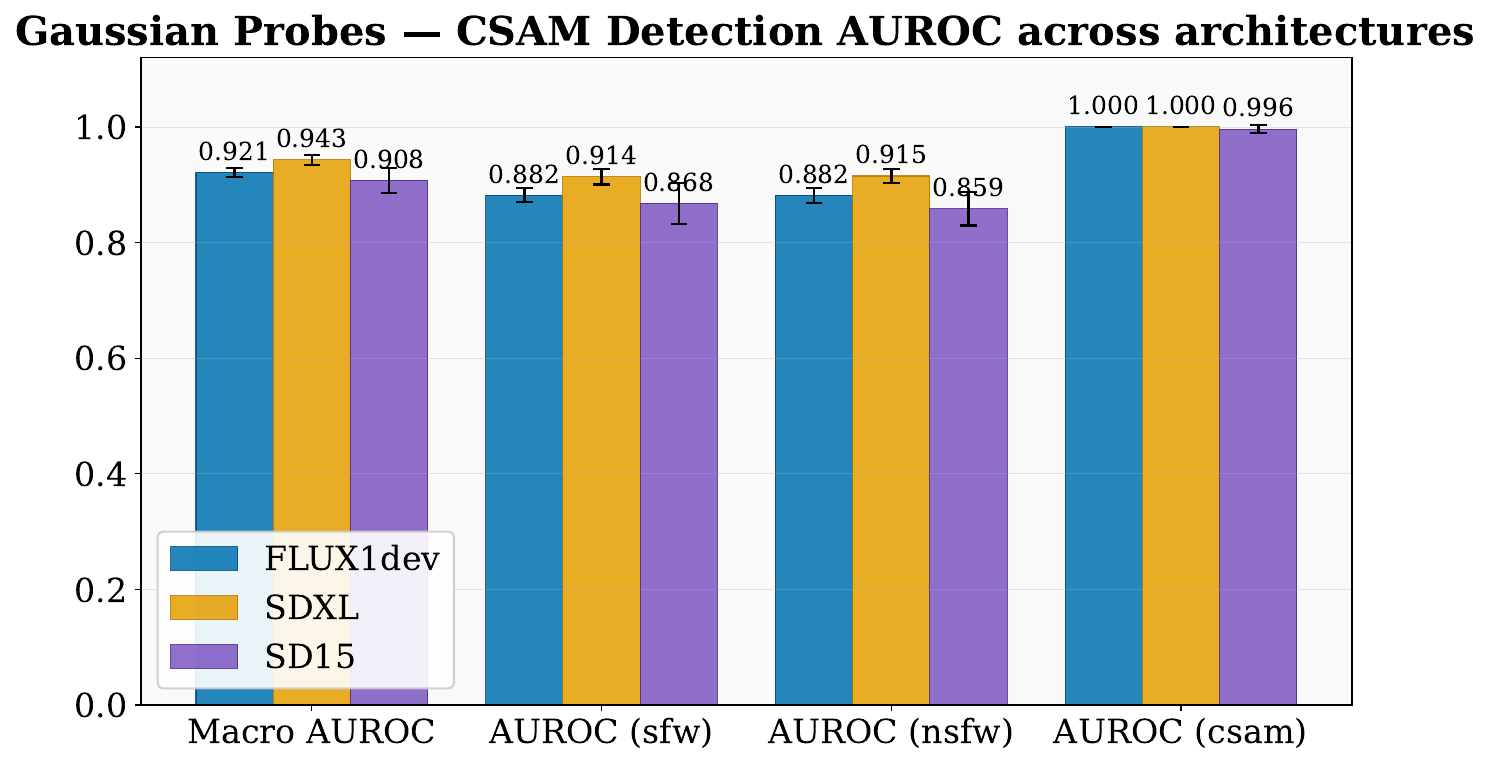}
        \caption{}
        \label{fig:panel_b}
    \end{subfigure}
    \caption{(a) Key insight: LoRA fine-tuning shifts the model’s feature distribution, separating benign and explicit outputs in latent space. Proposed Gaussian probing pipeline, which samples Gaussian noise, runs the diffusion model with LoRA, records intermediate activations across timesteps, and aggregates them into feature vectors without generating images. (b) Example results for all three architectures (SD 1.5, SDXL 1.0, FLUX.1-dev) for NSFW and CSAM detection in the wild. }
    \label{fig:pipeline}
\end{figure}

We call this the {\bf Evaluation without Generation} problem.  In this work, we study a concrete and narrower instance that arises in open ecosystems: screening LoRA adaptors for evidence of \emph{direct specialization} toward two sensitive content categories, adult sexual content (NSFW) and CSAM, without generating outputs. We focus on LoRAs because they are the primary unit through which specialization is created, packaged, and distributed. We answer this in the affirmative by shifting evaluation from output space to the model’s state. Rather than measuring what a model generates, we measure how fine-tuning changes its parameters or internal computation. This reframes capability evaluation as an inference problem over model state rather than outputs. %in these ecosystems: small, portable artifacts that can proliferate widely before any downstream use is observed.

We propose {\bf Gaussian probing}, a method that characterizes how a LoRA functionally perturbs a base diffusion model by measuring its internal responses to random Gaussian inputs. These responses provide a scalable, prompt-free signature of adaptor specialization without requiring image generation. We evaluate the method in two settings. First, in controlled experiments on in-house trained safe-for-work (SFW) and not-safe-for-work (NSFW) LoRAs spanning variation in datasets, styles, architectures, and training conditions. Second in a naturalistic setting using LoRAs collected from public platforms such as CivitAI and CSAM LoRAs accessed through authorized entities and in accordance with applicable laws, where heterogeneity, label noise, and shortcut opportunities better reflect deployment conditions. Across both regimes, Gaussian probing reliably detects harmful NSFW and CSAM specialization while remaining more robust than raw-weight baselines to superficial training artifacts. To our knowledge, {\em this is the first scalable, non-generative method for pre-distribution screening of CSAM-related specialization in user-uploaded generative models.}

% New citations 
% -https://www.runpod.io/case-studies/civitai-runpod-case-study
% - https://arxiv.org/pdf/2505.03859

%% file: capability.tex
\subsection{Defining the Governance Target and Scope}
\label{Sec:Formulation}

Open platforms face a practical question when users upload LoRA adaptors: does this warrant intervention before it is widely shared? Put differently, can an uploaded adaptor contribute to harmful downstream use? Answering that question requires some care, because capability is not an intrinsic property of a LoRA in isolation. Whether an adaptor is treated as harmful depends on how it is expected to be used and on what kind of behavior a platform is prepared to tolerate. For example, a platform might choose to permit adaptors with moderate adult-content propensity, while adopting a zero-tolerance rule for any minors-related propensity.

In full generality, a capability definition is difficult to identify. The relevant use distribution may be uncertain, institution-specific, or strategically shaped by users, and the policy itself is partly normative. In this work, we therefore study a narrower but still useful case of {\em direct harmful specialization}: whether a LoRA was directly fine-tuned on data drawn from a target harmful content category, specifically, CSAM. This does not capture every route by which harmful capability may arise, such as LoRA composition, model merging, or more complex downstream chaining. But it does capture an important and operationally meaningful slice of the broader governance problem, namely the one most directly tied to the uploaded artifact itself and most amenable to pre-distribution screening.

%This narrower focus still leaves a difficult auditing problem. First, as mentioned previously, observability is limited. 
%In the high-risk settings that motivate our study, e.g. CSAM identification, output-based evaluation is legally prohibited so the auditor cannot rely on generating and inspecting model outputs. 
%Second, screening must happen at repository scale. Platforms may need to assess large and constantly changing collections of uploaded adaptors, which rules out methods whose storage or compute requirements scale impractically with the full parameter dimension. Third, the signal used for screening must be robust. A useful auditor cannot rely on incidental features of how particular communities fine-tune models, nor on dataset-specific style cues that happen to correlate with harmful labels in one sample. Those signals may disappear under distribution shift and may be easy for adversaries to manipulate without changing the adaptor's underlying specialization.
Even under this narrower target, the auditing problem remains difficult. The auditor must detect harmful specialization from the uploaded adaptor and base model alone, under constraints of observability, scale, and robustness.
We formalize these constraints as the following methodological desiderata:

\begin{enumerate}[label=\textcolor{navy}{\textbf{(D\arabic*)}}]
\item{\em Non-generative observability.}
The representation must be computable from the uploaded adaptor and the base model alone, without requiring prompt design, output generation, or human review. 

\item{\em Scalability.}
A practical auditor must operate under realistic storage and compute constraints. First, the resulting representation must be much smaller than the full parameter dimension, so that a large repository of adaptors can be stored and processed in memory. Second, the cost of constructing the representation for a single adaptor must remain feasible at audit scale. %We therefore seek representations that can be computed directly from each adaptor with moderate per-adaptor cost, while keeping downstream classifier training efficient.

\item {\em Robustness.}
The representation should reflect direct harmful specialization rather than incidental features of how a LoRA was produced. We identify three types of signal that \textit{can} be predictive and describe which are the ideal ones to utilize: 

\begin{itemize}[leftmargin=*]
    \item \textbf{development artifacts} such as rank, learning-rate schedules, or update magnitude
    \item \textbf{dataset identity} such as stylistic or distributional peculiarities of a particular training image datasets
    \item \textbf{content signal} or information tied to the underlying harmful category that persists across datasets, fine-tuning runs, and implementation choices
\end{itemize}

This distinction matters for generalization and robustness to adversarial changes.%A detector that relies on superficial artifacts may perform well in-distribution while failing under even simple adaptive attacks.
%The representation should reflect direct harmful specialization rather than incidental features of how a LoRA was produced. In particular, it should avoid relying on \emph{development artifacts}, such as rank, learning-rate schedules, or update magnitude, and it should not collapse onto \emph{dataset identity}, such as stylistic or distributional peculiarities of a particular training corpus. The ideal signal is instead \emph{content signal}: information tied to the underlying harmful category that persists across datasets, fine-tuning runs, and implementation choices. This matters not only for generalization, but also for robustness to adversarial attacks. A detector that relies on superficial artifacts may perform well in-distribution while failing under even simple adaptive manipulations.
\end{enumerate}
%Taken together, these desiderata narrow the space of acceptable screening methods. %The problem is not merely to separate harmful from benign LoRAs, but to do so from restricted evidence, at repository scale, and in a way that remains meaningful under heterogeneous and potentially strategic fine-tuning practices.
\subsection{Formalizing the Screening Task}
Fix a harmful content category \(c\) and a pretrained base diffusion model
\(f_{\theta_{\mathrm{base}}}\). For a LoRA adaptor \(\Delta\), let
\(
f_{\theta(\Delta)} := f_{\theta_{\mathrm{base}}+\Delta}
\)
denote the adapted model. For each adaptor \(\Delta\), let
\(
y(\Delta) \in \{0,1\}
\)
indicate whether \(\Delta\) was directly fine-tuned on data drawn from category \(c\).

The auditor observes \((\theta_{\mathrm{base}}, \Delta)\), but not the underlying training data and not any generated outputs. The auditing task of interest is given labeled adaptors \(\{(\Delta_i,y(\Delta_i))\}_{i=1}^n\), predict \(y(\Delta)\) from weights alone. We consider screening rules of the form
\(
g \circ \Phi,
\)
 where \(\Phi(\Delta;\theta_{\mathrm{base}}) \in \mathcal H,
\) 
is a fixed representation map and \(g:\mathcal H \to \{0,1\},
\) is a learned classifier. The auditor solves
\[
\hat g_\Phi
\in
\arg\min_{g}
\frac{1}{n}\sum_{i=1}^n
\ell\!\left(y(\Delta_i), g(\Phi(\Delta_i;\theta_{\mathrm{base}}))\right)
+\lambda \|g\|^2 .
\]
The central technical question is therefore how to construct a representation
\(\Phi(\Delta;\theta_{\mathrm{base}})\) that recovers direct harmful specialization from restricted, non-generative evidence while satisfying the scalability and robustness desiderata.

% We formulate the auditor's task as follows. Let $\theta_{\texttt{base}} \in \mathbb{R}^d$ denote the pretrained base model and $\{\Delta_i, y_i\}_{i=1}^n$ denote a collection of $n$ LoRA adaptors  where $y_i \in \{0, 1\}$ encodes whether each LoRA has been specialized for unsafe content ($y_i = 1$) or not ($y_i = 0$). 
% The auditor's objective is to learn a classifier $g \circ \Phi$ operating on a structured representation of each adaptor, where $\Phi:\mathbb{R}^d \times \mathbb{R}^d\rightarrow \mathcal{H}$ is a feature map that extracts representations from each adaptor relative to the base model, and $g: \mathcal{H} \rightarrow \{0,1\}$ is the classifier trained on these representation. Intuitively, the auditor must learn to distinguish harmful from benign adaptations based on how each LORA modifies the base model. Formally, the general task can be formulated as,
% \begin{equation}
%     \min_{g, \Phi} \frac{1}{n}\sum_{i=1}^n \ell (y_i, g(\Phi (\Delta_i; \theta_{\texttt{base}}))) + \lambda \|g\|^2. 
% \end{equation} 

\section{Gaussian Probing}

We instantiate the representation map \(\Phi(\Delta;\theta_{\mathrm{base}})\) by probing the adapted model on a reference ensemble of Gaussian latent states. This representation summarizes how the adaptor changes the computation implemented by the base model.
%The central challenge is to construct a non-generative representation of a LoRA adaptor that is scalable, robust, and informative about direct harmful specialization. A raw-weight representation treats \(\Delta\) as a point in parameter space. Our goal, however, is not to describe the adaptor in raw coordinates, but to characterize how it changes the computation implemented by the base model. We therefore build a representation by probing the adapted model on a reference ensemble of Gaussian latent states.

Let \(H\), \(W\), and \(C\) denote the height, width, and number of channels of the diffusion model input, and define \(\bar d = HWC\). Let 
\(
\theta(\Delta) = \theta_{\texttt{base}} + \Delta
\)
denote the adapted parameters. We draw \(m\) i.i.d. Gaussian probes
\[
\nu_j \sim \mathcal{N}(0,I_{\bar d}), \qquad j=1,\dots,m.
\]

Each probe \(\nu\) is propagated through the diffusion process using the adapted model \(f_{\theta(\Delta)}\) for \(T\) denoising steps. During this process, we extract intermediate hidden representations from a fixed layer, or fixed set of layers, of the denoising network.\footnote{In practice, we extract activations from one or more designated internal layers of the denoising network, such as a mid-block feature map of the U-Net in Stable Diffusion 1.5.} Let
\(
H^{(T)}(\Delta;\nu), H^{(T-1)}(\Delta;\nu), \dots, H^{(1)}(\Delta;\nu)
\)
denote the resulting hidden representations across diffusion timesteps.
For a single probe \(\nu\), we aggregate these hidden states into a probe-specific feature
\[
\bar H(\Delta;\nu)
=
\psi\left(H^{(1)}(\Delta;\nu),\dots,H^{(T)}(\Delta;\nu)\right),
\]
where \(\psi\) is a fixed pooling map across timesteps and, when relevant, across layers, spatial positions, and channels. In the simplest case, \(\psi\) is the timestep average
\[
\bar H(\Delta;\nu)
=
\frac{1}{T}\sum_{t=1}^T H^{(t)}(\Delta;\nu).
\]

The population object underlying Gaussian probing is the probe functional
\begin{equation}
\label{eq:pop_probe}
\Psi(\Delta):=
\mathbb{E}_{\nu \sim \mathcal{N}(0,I_{\bar d})}\big[\bar H(\Delta;\nu)\big].
\end{equation}
This function summarizes how the adapted model responds, on average, to a reference ensemble of Gaussian latent states. Our empirical representation is the corresponding Monte Carlo estimator
\begin{equation}
\label{eq:emp_probe}
\Phi(\Delta;\theta_{\texttt{base}})
:=
\widehat{\Psi}_m(\Delta)
=
\frac{1}{m}\sum_{j=1}^m \bar H(\Delta;\nu_j).
\end{equation}
As we will illustrate, %under mild moment assumptions, \(\widehat{\Psi}_m(\Delta)\) converges almost surely to \(\Psi(\Delta)\) as \(m \to \infty\). In this sense, 
this can be interpreted as the expected pushforward of the LoRA perturbation through the model’s denoising dynamics under the native latent distribution. 
%Gaussian probing estimates a stable, finite-dimensional summary of the adapted model's internal response profile on the model's native Gaussian state space.

\subsection{Motivation for Gaussian Probing}
%The screening task is fundamentally about behavior, not weights. 
What the auditor ultimately cares about is whether a LoRA changes the model in a way that supports harmful downstream use. In principle, the most direct object of study would therefore be the distribution of model outputs under some reference input ensemble. But in the settings that motivate this work, outputs cannot be generated. Therefore our approach replaces outputs with internal activations: rather than asking what the adapted model renders, we ask how the adaptor changes the model’s internal response profile on the diffusion process’s native Gaussian state space.

This substitution is meaningful only if intermediate activations carry semantically relevant information. Prior work suggests that they do. In particular, diffusion models appear to organize high-level concepts in internal latent representations that remain coherent across denoising steps, and those representations can be used to decode or steer semantic attributes of generation~\citep{kwon2022diffusion}.  We leverage the same structure for a different purpose. Rather than steering outputs, we use internal responses to detect whether a LoRA systematically shifts the model toward a harmful specialization.

This intuition can be made more precise through a local linearization.
Fix a timestep \(t\) and probe \(\nu\), and suppose the extracted hidden state is differentiable with respect to the model parameters. Then for a LoRA-induced perturbation 
\(\Delta\),
\begin{equation*}
    H^{(t)}(\Delta;\nu) - H^{(t)}(0;\nu)
    =
    D_\theta H^{(t)}(\theta_{\mathrm{base}};\nu)[\Delta]
    + o(\|\Delta\|),
\end{equation*}
where \(D_\theta H^{(t)}\) denotes the Fr\'echet derivative of the hidden state mapping at the base parameter configuration.
Equivalently, after vectorization, 
\begin{equation*}
    H^{(t)}(\Delta;\nu) - H^{(t)}(0;\nu)
    \approx
    J_t(\theta_{\mathrm{base}};\nu)\,\mathrm{vec}(\Delta),
\end{equation*}
where \(J_t\) is the Jacobian matrix representing the sensitivity of the internal representations to parameter changes.

Under this view, Gaussian probing is not an arbitrary feature extractor. It measures how the LoRA is pushed forward into activation space by the denoising dynamics of the base model. Averaging over Gaussian probes emphasizes parameter directions that actually affect the model’s computation on its native latent state space, rather than treating all directions in weight space as equally meaningful. Because LoRA updates are low-rank, the resulting activation shift is constrained to a locally low-dimensional subspace. This gives a principled reason to expect that harmful and benign specializations may be distinguishable by relatively simple decision rules, including linear classifiers, if their induced activation responses differ systematically.

We use Gaussian noise as the reference ensemble for two reasons. First, it is native to the diffusion process itself, so it provides a prompt-free way to interrogate the model’s denoising dynamics without making assumptions about future user prompts. Second, the probes are sampled independently of the fine-tuning pipeline. Because they do not inherently encode dataset metadata or rank selection, these responses provide a cleaner signal for desideratum~{\color{navy} \bf (D3)}. If a classifier built on Gaussian probe responses separates harmful from benign adaptors, that separation will likely come from how the LoRA functionally perturbs the model on the reference state space instead of superficial traces of how the LoRA was produced. In this sense, Gaussian probing is designed to privilege content-relevant functional signal over incidental training artifacts.

Since \(\Phi(\Delta;\theta_{\mathrm{base}})\) is a Monte Carlo estimator of
\(\Psi(\Delta)\), standard law-of-large-numbers arguments imply consistency under mild moment assumptions; we state this formally in Appendix~\ref{App:Theory}. The central question is therefore empirical: whether the induced shift in this representation is both sufficiently large and sufficiently well-estimated at finite sample sizes to support reliable screening in practice.
% Rather than assuming a margin exists, we treat its existence as our primary research question. While theory suggests that harmful specialization should shift the model’s internal representations in a predictable way, the practical ``width'' of that shift is an empirical matter. We show that this margin is indeed sufficient for classification in Sections~\ref{Sec:Explore} and~\ref{Sec:Wild}.

\paragraph{Gaussian probing satisfies the desiderata.}
Computing $\Phi(\Delta; \theta_{\mathrm{base}})$ requires only forward passes through the adapted model on Gaussian inputs: no outputs are decoded, no images are rendered, and no prompt is designed or reviewed, satisfying {\color{navy} \textbf{(D1)}}. The computation requires $O(mT)$ forward passes, is parallelizable across the audit set, and produces a representation of fixed dimension $|\mathcal{H}|$ regardless of the adaptor's rank or structure. This representation is orders of magnitude smaller than the full parameter dimension $d$, satisfying {\color{navy}\textbf{(D2)}}. Finally, because the representation captures the functional effect of $\Delta$ on intermediate activations rather than its footprint in weight space, in principle, two adaptors with identical functional behavior but different weight norms or training configurations will produce similar probe activations, while two adaptors with different specialization will not, satisfying {\color{navy}\textbf{(D3)}}. The remainder of the paper evaluates whether it succeeds in doing so empirically.

%% file: sections/adult_sexual_content_detection.tex
We evaluate how well the representation \(\Phi\) satisfies the desiderata introduced in Section~\ref{Sec:Formulation}. We begin with a controlled setting, using LoRAs trained on SFW and adult pornography (NSFW) content to isolate separability and signal quality. We then consider a naturalistic setting using LoRAs from CivitAI to study robustness to training artifacts.

We compare Gaussian probing to classifiers built on raw weight representations, including random projections of weight matrices. While these methods are simple and non-generative, they treat all directions in parameter space as equally meaningful and may therefore rely on incidental features of the training process rather than underlying content specialization.

We evaluate three questions:
(1) Does harmful specialization induce sufficient separation in representation space?
(2) Do representations capture content signal rather than dataset identity or artifacts?
(3) Are these properties preserved in naturalistic settings and under adversarial conditions?

\subsection{Controlled Study}
\paragraph{Experimental Setup} %First, we describe our pipeline for producing our dataset of LoRAs that we will analyze. 
To produce our SFW LoRAs, we use six different datasets: COCO2017~\citep{lin2014microsoft}, Flickr30K~\citep{plummer2015flickr30k}, Conceptual Captions 12M~\citep{changpinyo2021conceptual}, LAION-Aesthetics~\citep{schuhmann2022laion}, OpenImages~\citep{kuznetsova2020open}, Unsplash-Lite~\cite{unsplashlite}, and Wikiart~\cite{saleh2015largescaleclassificationfineartpaintings}. Together, these cover a broad range of benign content, including images of humans (faces, poses, non-explicit suggestive content), landscapes, objects, art, and fashion. We also  use three adult pornography datasets in this task: NSFW Video Still~\citep{morelli2016nsfw}, Danbooru2023~\citep{nyanko2023danbooru}, Amaye15~\citep{amaye2025objectsegmentation}. All datasets are filtered using Thorn’s Safer hashing and classification technology to ensure removal of CSAM.\footnote{\url{https://safer.io/solutions/}} For each adult dataset we create 10 different random samples, along with additional subsets (e.g., cartoon-only and varying levels of nudity) to support our investigation into the kind of signal being encoded by these representations. 

We train 1,000 LoRAs per class (SFW and NSFW), varying dataset, sample, rank, learning rate, modules, training steps, and data size. This randomization ensures that models cannot rely on spurious correlations tied to specific training configurations, forcing representations to capture underlying content signal. In order to avoid the failure mode we already identified, we intentionally choose to remove these potential confounds. For Gaussian probing, we sample 1024 probes for each model and for raw weights we use a projection dimension of 256 for every layer.

We evaluate across Stable Diffusion 1.5, SDXL 1.0, and FLUX.1-dev, spanning a range of model sizes and architectures from 860M to 12B parameters, including both U-Net and transformer-based diffusion models. This diversity allows us to study the impact of architecture and scale, as well as the role of layer selection for representation extraction. In particular, we examine which layers carry the most semantic signal, especially for SDXL and FLUX where prior guidance is limited.

We report accuracy, precision, recall, F1, FPR, and FNR under two settings: (1) standard 5-fold cross-validation, and (2) a leave-dataset-out (LDO) setting to test conceptual generalization. For this particular application, we are most interested in the AUROC, the precision and the false positive rate (FPR). Thus, these are the metrics we report throughout the plots in the main paper. All other metrics are reported in Appendix~\ref{app:structured_cv}.

It is important to note that this controlled study eliminates many of the confounds which make this task inherently more difficult in the wild. These confounds include: which base checkpoint the LoRA was finetuned from, biases in which layers of the model are finetuned, and each individual LoRA being trained on a different dataset. As a result, we expect that the performance we see in this controlled setting will be higher than what we would see in the wild. Nevertheless, it more cleanly allows us to investigate the questions we outlined without the confounds playing a role.

\begin{figure}
    \centering
    \includegraphics[width=\linewidth]{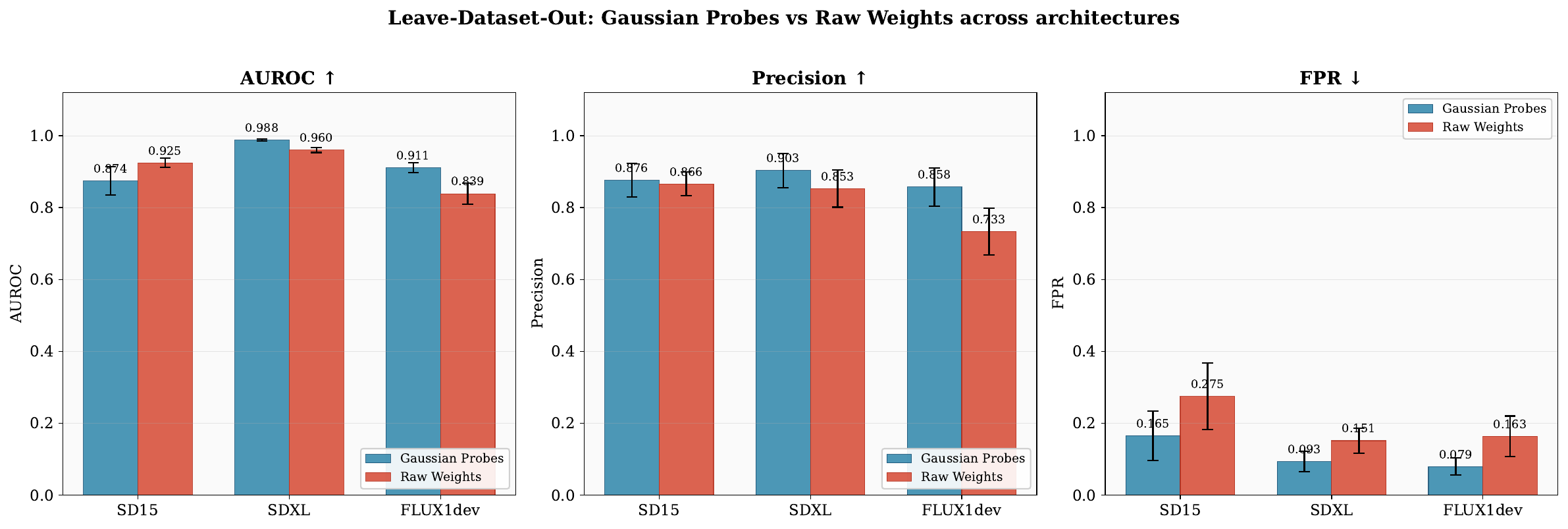}
    \caption{Average AUROC, precision, and FPR across 5 fold CV fo all architectures on the leave-dataset-out evaluation. Randomly projected weights outperform Gaussian probing depending on the model, but we identify that this is due to dataset identity signal in Section~\ref{sec:content_vs_dataset}, which does not satisfy our desiderata. Meanwhile, Gaussian probing still performs well while primarily leveraging the content signal in accordance with our robustness desiderata.}
    \label{fig:controlled_ldo}
\end{figure}

\subsubsection{Separability of SFW and NSFW Specialization}
\label{sec:separation}
We begin with SD 1.5, as this is the simplest of the models and the most studied in the interpretability literature. Both representations achieve strong performance under standard cross-validation~(Figure~\ref{fig:sd15_standard_cv}), with Gaussian probing showing slightly lower FPR. The performance is exceptionally strong and we believe shows the ceiling of what is possible when all confounds are controlled. A more appropriate metric is the leave-dataset-out (LDO) setting. Given that we are leaving one entire SFW and NSFW dataset out during training, we get a better sense of \textit{conceptual generalization}. Randomly projected weights generalize better than gaussian probing but both do well~(Figure~\ref{fig:controlled_ldo}). However, this apparent advantage relies on stable dataset-level artifacts. As we illustrate in the next section, in settings where such artifacts can be manipulated, these representations degrade under even simple transformations, whereas Gaussian probing targets invariant functional signal.

For SDXL 1.0 and FLUX.1-dev, we see similarly high performance in the standard CV evaluation~(Figure~\ref{fig:sdxl_standard_cv} and Figure~\ref{fig:flux1dev_standard_cv}). For LDO evaluation, we find that for SDXL 1.0, probes generalize better. Meanwhile, for FLUX.1-dev, raw weights generalize better. In the next section, we investigate whether this generalization comes from signal that satisfies our auditor desiderata or if it is due to spurious signals that will inevitably undermine robustness. Overall, these results demonstrate that both representations we have presented are capable of distinguishing SFW trained LoRAs from NSFW trained LoRAs. We ensemble probes from both latent diffusion and text encoder modules across all three architectures in Appendix~\ref{app:controlled_module_ensemble} show this is essential for detecting text-encoder-only finetunes for SD 1.5.

% \begin{figure}[t]
%     \centering
%     \includegraphics[width=\linewidth]{Figures/adult_sexual_content/method_comparison_ldo_SD15_bars.pdf}
%     \caption{Caption}
%     \label{fig:sd15_ldo_bars}
% \end{figure}

\begin{figure}[t]
    \centering
    \includegraphics[width=\linewidth]{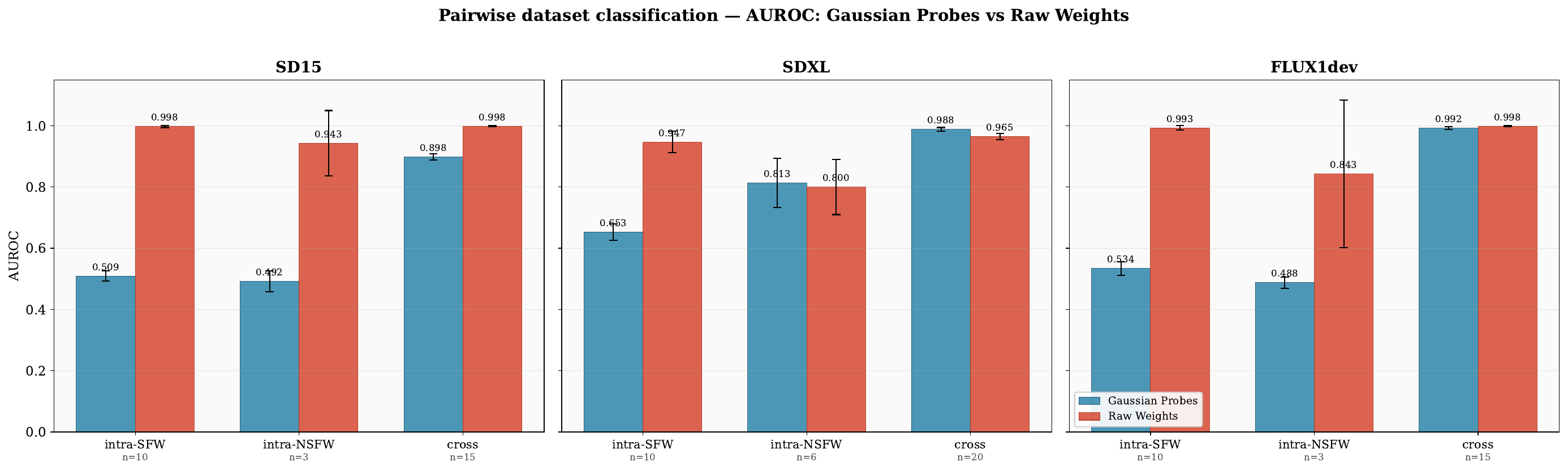}
    \caption{Average AUROC values for intra-SFW, intra-NSFW, and cross-class dataset pairs for both raw weights and Gaussian probing. For SD 1.5 and FLUX.1-dev Gaussian probing clearly is relying on content signal given both intra-SFW and intra-NSFW AUROC are near random while cross-class AUROC is high. This suggests that Gaussian probing encodes a signal that is useful for distinguishing the two concepts, without relying on lower-level dataset-specific properties.}
    \label{fig:contentvsdataset}
\end{figure}

\subsubsection{Content Signal vs. Dataset Identity}
\label{sec:content_vs_dataset}
We next evaluate how well the representations satisfy the robustness desideratum, in particular whether they capture content signal rather than dataset identity. The leave-dataset-out results in Section~\ref{sec:separation} left open whether raw weights' advantage reflects genuine signal or exploitable artifacts. This section shows it's the latter.

To isolate the source of this signal, we construct a pairwise dataset classification task over SFW and NSFW LoRAs. For each pair of datasets, we train classifiers using both randomly projected weight and Gaussian probing features, and measure performance within and across classes. A representation that primarily encodes dataset identity will achieve high and consistent performance across intra-SFW, intra-NSFW, and cross-class comparisons, whereas a representation that captures content signal should perform well only on cross-class distinctions. 

First, we examine whether simple training artifacts could explain the observed separability. In particular, we compare the norms of LoRAs across and within classes and find them to be similar, reducing the likelihood that magnitude alone drives classification.  Additionally, given that we randomly selected across all training hyperparameters we capture the vast space of possible configurations. Thus, we are less concerned that the representations are separable by training artifact information alone. 

A representation that encodes dataset identity will separate any two datasets regardless of class. A representation that encodes content signal will separate only across class (SFW vs. NSFW), not within. Figure~\ref{fig:contentvsdataset} shows which regime each method is in. Overall, we find that Gaussian probes exhibit a stronger reliance on content signal. This is most evident in SD 1.5 and FLUX.1-dev, where projected raw weights can distinguish between datasets regardless of class, indicating sensitivity to dataset-specific artifacts rather than underlying content. Interestingly, for SDXL 1.0 where Gaussian probes do look like they use a combination of dataset identity and content signal it outperforms the randomly projected weights in the LDO evaluation. These results demonstrate that while Gaussian probes tend to perform worse on the LDO evaluation than randomly projected weights, they are using the ideal signal. Thus, we believe they are more robust and satisfy the auditor desiderata we outlined.

\subsection{In the Wild Study}
\subsubsection{Are we relying on development artifacts \& community norms?}
\label{sec:whyrobust}

The focus on classifying purely using content signal over both dataset identity and training artifacts may seem unnecessarily restrictive. Instead, one might argue for leveraging all available signals to maximize classification accuracy, regardless of whether it is based on the content or the training development artifacts. If such artifacts are stable and consistent in the wild on platforms such as HuggingFace, CivitAI, or CivArchive, why not exploit them? This line of reasoning seems reasonable, especially for high-stake domains like CSAM, where errors can have enormous costs. 

The problem is that relying on these artifacts introduces a fundamental vulnerability. Unlike content signals, which require meaningful manipulation to the model's behavior, training artifacts are superficial. Adversaries can easily manipulate them, mimicking benign training patterns, without altering what the LoRA actually does. As a result, defenses that rely on such artifacts are likely to degrade quickly once adversaries adapt.

\paragraph{Experimental Setup} To illustrate this vulnerability, we conduct a small experiment using LoRAs sourced from CivitAI. We collect 1,118 SDXL 1.0 LoRAs from the CivitAI platform, 493 of which are trained on SFW content and 525 on adult sexual content. Assignment of LoRAs into SFW vs. NSFW classes is based on a combination of LoRA metadata, an NSFW classifier applied to associated gallery images, and user-reported flags.

We analyze metadata across both classes and observe systematic differences in training practices. Specifically, the median number of training steps, the median learning rate, and the median LoRAs rank are all higher for the NSFW LoRAs (Table~\ref{tab:lora-config-comparison}), suggesting larger update magnitudes. We perform k-fold cross-validation using random train-test splits of our data, comparing the random projections classifier with the Gaussian probes classifier. We find that both perform equally well, initially suggesting that perhaps the main benefit of Gaussian probing is its memory efficiency.

\begin{table}[t]
\centering
\begin{tabular}{lcc}
\toprule
\textbf{Field} & \textbf{SFW} & \textbf{NSFW} \\
\midrule
Training steps     & 1616 \scriptsize{[1000--3342]} \tiny{($n=363$)} & 1776 \scriptsize{[864--4000]} \tiny{($n=349$)} \\
Learning rate      & 0.0001 \scriptsize{[0.0001--0.0005]} \tiny{($n=359$)} & 0.0002 \scriptsize{[0.0001--0.0005]} \tiny{($n=346$)} \\
Network dim        & 32 \scriptsize{[16--66]} \tiny{($n=376$)} & 64 \scriptsize{[32--256]} \tiny{($n=520$)} \\
Network alpha      & 16 \scriptsize{[8--64]} \tiny{($n=375$)} & 32 \scriptsize{[16--256]} \tiny{($n=520$)} \\
LoRA rank          & 32 \scriptsize{[16--70]} \tiny{($n=418$)} & 64 \scriptsize{[32--256]} \tiny{($n=546$)} \\
Num.\ train images & 360 \scriptsize{[202--1482]} \tiny{($n=361$)} & 590 \scriptsize{[220--1260]} \tiny{($n=349$)} \\
Num.\ epochs       & 10 \scriptsize{[8--14]} \tiny{($n=361$)} & 10 \scriptsize{[8--20]} \tiny{($n=349$)} \\
Num.\ tensors      & 792 \scriptsize{[792--792]} \tiny{($n=508$)} & 792 \scriptsize{[792--834]} \tiny{($n=547$)} \\
\bottomrule
\end{tabular}
\caption{Training configuration comparison between SFW and NSFW SDXL 1.0 LoRA models (median [Q25--Q75]).}
\label{tab:lora-config-comparison}
\end{table}

However, further analysis reveals that classifiers trained on raw weights are brittle. In fact, we found that if we simply train a classifier on feature vectors comprised of the Frobenius norm of each layer, this performed comparably to the random projections method. The classifier suggested that it was using the magnitude as a shortcut, given the differences in training procedures shown by the metadata. 

This leaves a clear vulnerability for an adversary to manipulate by simply equalizing the norms of their NSFW (or CSAM) LoRAs to match those of publicly available SFW LoRAs. We simulate this attack with our sample of SDXL LoRAs. We train classifiers on the unnormalized LoRAs and then test them on LoRAs where each layer's norm is set to unit Frobenius norm. We see that the raw weights classifier experiences a significant drop in performance across most metrics, particularly AUROC, AUPRC, precision, and FPR. FPR in particular increases by 20\%. Meanwhile the Gaussian probes classifier remains the same or even improves slightly (depending on the metric) on these normalized LoRAs at test time (Figure~\ref{fig:magnitude_robustness}).

\begin{figure}[t]
    \centering
    \includegraphics[width=\linewidth]{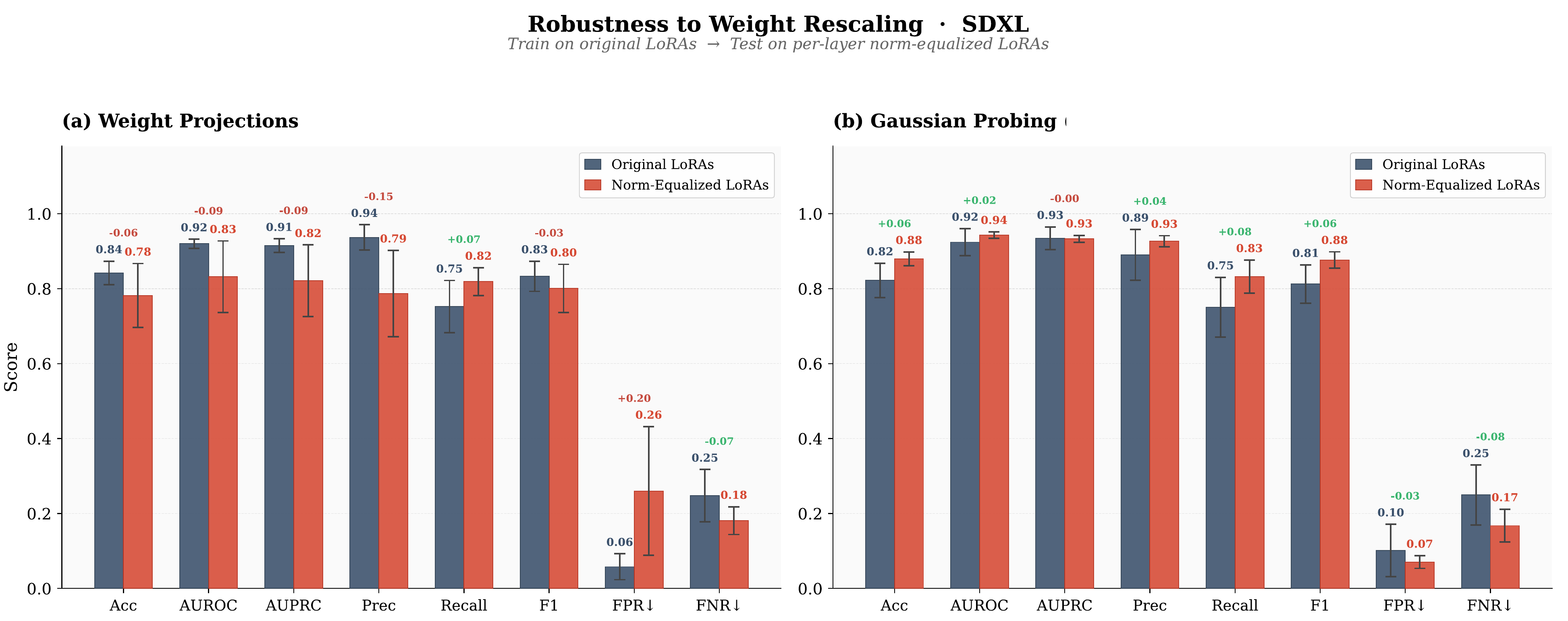}
    \caption{\textbf{Robustness to weight rescaling under an adversarial setting.} Both methods are trained on unmodified CivitAI LoRAs and evaluated on LoRAs whose per-layer weight deltas have been rescaled to unit Frobenius norm, removing magnitude information while preserving functional behavior. (a) Classification using random projections of weight matrices degrades substantially across all metrics, with AUROC dropping from 0.920 to 0.832 and confidence interval widening from ±0.012 to ±0.095. (b) Gaussian probing is robust to the same manipulation, with AUROC increasing slightly from 0.924 to 0.943 and variance decreasing. This confirms the importance of relying on content signal over training artifacts differences.}
\label{fig:magnitude_robustness}
\end{figure}

%% file: sections/child_sexual_abuse_content_detection.tex
\begin{figure}[t]
    \centering
    \includegraphics[width=\linewidth]{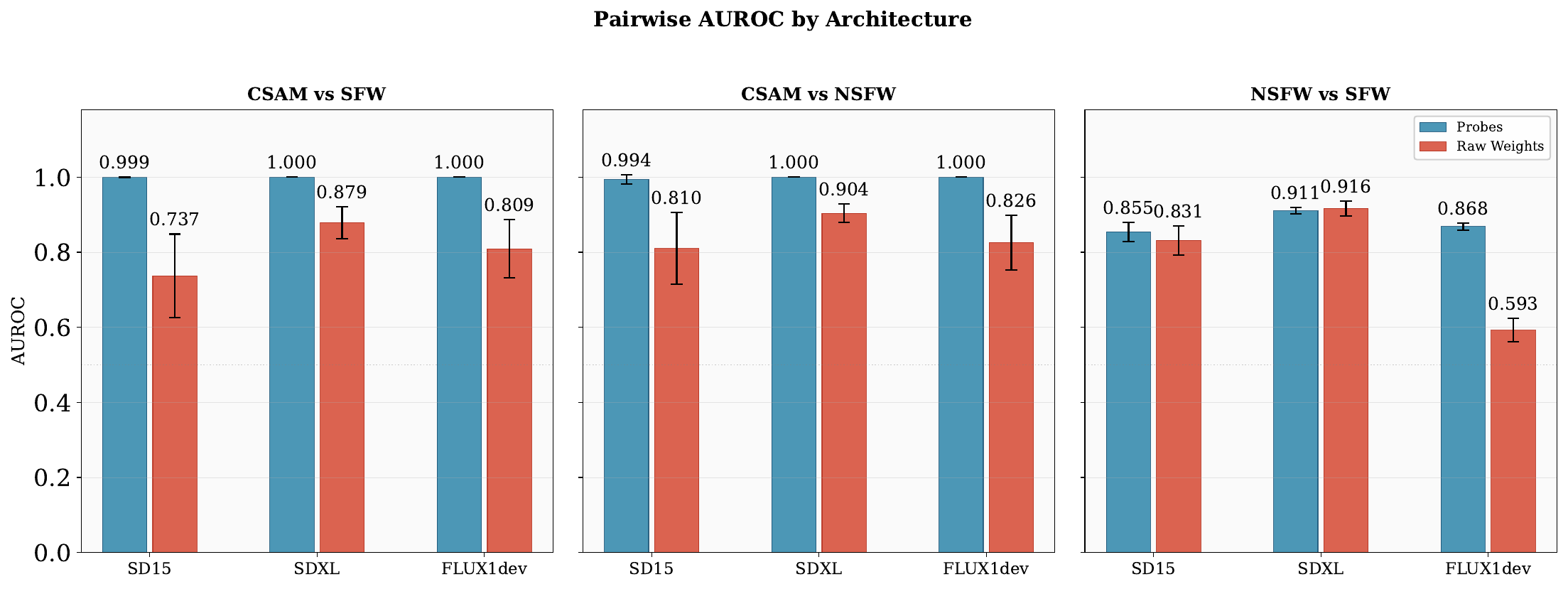}
    \caption{AUROC across our three models (SD 1.5, SDXL 1.0, and FLUX.1-dev) computed for each pair of classes (SFW vs. CSAM), (NSFW vs. CSAM), and (SFW vs. NSFW). Gaussian probing separates CSAM from both SFW and NSFW across all three architectures. Raw weights fail on small-sample CSAM detection (AUROC 0.68-0.87). NSFW vs CSAM should be viewed as the most difficult of the three. Although, raw weights do perform comparably on some of the folds, there is high variance. Meanwhile the methods perform more comparably when looking at SFW vs NSFW.}
    \label{fig:csam_auroc}
\end{figure}

In this section, we focus on the motivating application of this work: detecting LoRAs specialized for generating child sexual abuse material using Gaussian Probing. Having demonstrated the effectiveness of the different representations for LoRA classification in satisfying our desiderata of non-generative, scalability, and robustness (Sections \ref{sec:content_vs_dataset} and \ref{sec:whyrobust}), we focus on applying our algorithm on separating between LoRAs specialized to generate SFW, NSFW, and CSAM content. We did not possess or generate any CSAM data or CSAM-specific LoRAs, nor did we train any LoRAs on such material; all such data remained solely with authorized entities and handling of the data was done in accordance with applicable laws and organizational safeguards. We intentionally omit further details to reduce misuse risk and preserve operational security. We discuss our reasoning for omitting further details in Section~\ref{sec:ethics}.

\paragraph{Experimental Setup} For this section, our goal is to simulate deployment of this approach in the wild. Thus, we source around 1,000 LoRAs from CivitAI for the SFW and NSFW classes. Additionally, for the purposes of this research, we accessed CSAM-specialized LoRAs through authorized entities, without handling underlying CSAM data and in compliance with applicable laws and organizational safeguards. We discuss our choice to intentionally omit all other details regarding access to these LoRAs in the Ethical Considerations section of the paper. For SD 1.5, our sample size consists of 18 CSAM-specialized LoRAs, for SDXL 1.0 our sample size is 34 CSAM-specialized LoRAs, and for FLUX.1-dev our sample size is 74 CSAM-specialized LoRAs. We test the performance of the different representations in this prediction task across the same three models as in the previous section. For Gaussian probing, we sample 512 probes for SD 1.5 and SDXL 1.0. Due to computational and operational security constraints surrounding our CSAM-specialized FLUX.1-dev LoRAs we only use 4 probes. For raw weights we use a projection dimension of 256 for every layer.  Given that we no longer are controlling for any of the confounds we expect to see in the wild, we expect our method to perform well but less strong as our controlled experiments. As in Section~\ref{Sec:Explore}, we focus on AUROC, precision, and FPR as our primary metrics of interest. We report additional metrics in the Appendix.

\paragraph{Pairwise Detection} Across all three models and classes, Gaussian probing stands out as the better representation for detection. Raw weight projections, given their high dimensionality, struggle to classify CSAM-specialized LoRAs given their small sample size. Meanwhile, Gaussian probing is effective at correctly classifying all CSAM-specialized LoRAs (Figure~\ref{fig:csam_auroc} and Figure~\ref{fig:csam_auprc})). We view these results as promising for the effectiveness of Gaussian probing for our motivating application. At the same time, these results should be interpreted as evidence of separability under the current data conditions, rather than as a claim of near-perfect detection in deployment settings.

\begin{figure}
    \centering
    \includegraphics[width=\linewidth]{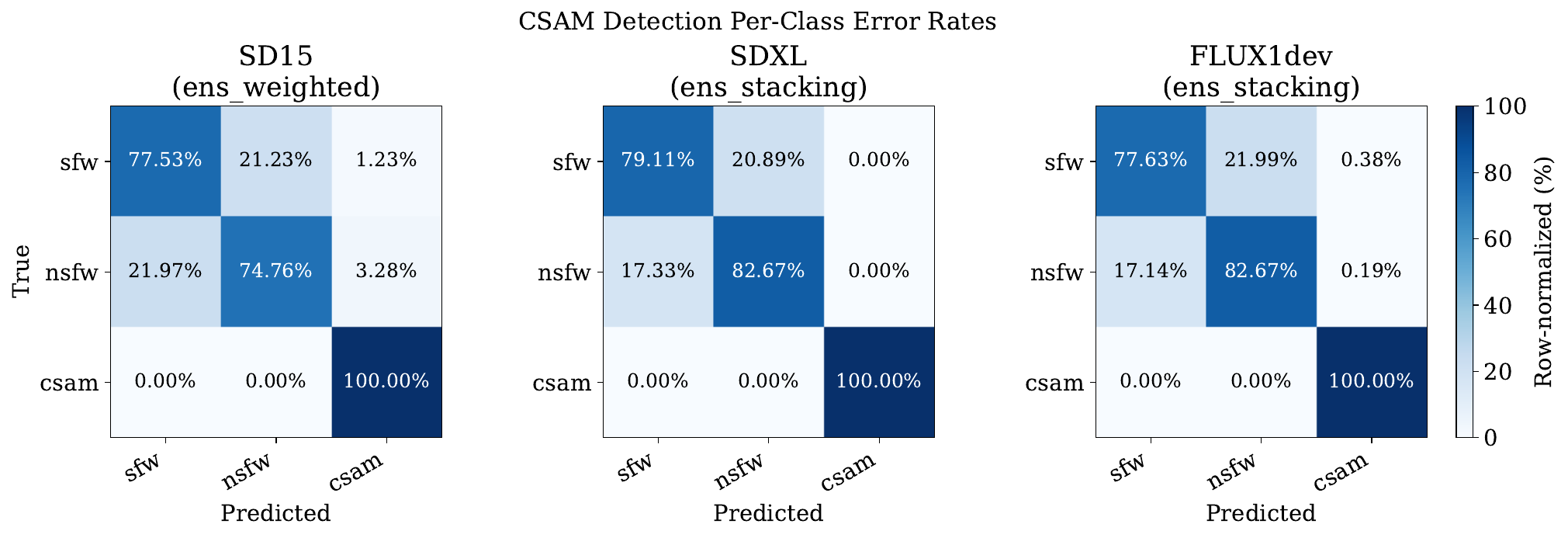}
    \caption{CSAM detection error rates by true class, across architectures averaged over 5 folds. Gaussian probing achieves 100\% recall on CSAM across SD 1.5, SDXL, and FLUX.1-dev (bottom row), with under 1\% of SFW LoRAs misclassified as CSAM for SD 1.5 and SDXL. The 2–4\% of NSFW LoRAs predicted as CSAM may reflect genuine CSAM contamination in CivitAI uploads rather than classifier errors, and warrant follow-up inspection about which of the two they are.}
    \label{fig:csam_errors}
\end{figure}

\paragraph{Per-Class Results and Errors} Within our results, we investigate the distribution of errors made by our probe classifier. In particular, we are interested in understanding how many SFW and NSFW LoRAs are classified as being CSAM-specialized LoRAs. While SFW LoRAs can most likely be considered true errors, it is possible that some of the NSFW LoRAs that are classified as being CSAM-specialized were in fact trained with some amount of CSAM data during finetuning. We find that across model architectures, probes classify all CSAM models correctly and have low false positive rates for classifying both SFW and NSFW models as CSAM. We see the highest FPR for FLUX.1-dev models at 4.21\% of SFW models being classified as CSAM specialized and 1.90\% of NSFW models being classified as CSAM (Figure~\ref{fig:csam_errors}). We view the confusion between SFW and CSAM as true errors, while the NSFW models could in fact be undetected CSAM-specialized models, warranting further inspection to distinguish true errors from contamination.

\paragraph{Design Choices} In order to maximize the performance of Gaussian probing, for we needed to ensemble predictions across multiple layers. We show in Appendix~\ref{app:layer_choice_itw} that without doing so Gaussian probing performs worse. While not surprising given the scale of these models, it does point to layer choice being an important consideration for applying Gaussian probing to future architectures. For SD 1.5, in accordance with existing literature on the existence of the \textit{h-space}~\cite{kwon2022diffusion}, we used the representation from after the mid-block in the U-Net as a starting point, combined with additional layers. Based on this insight, we started with a mid-block representation for all three models. For SDXL 1.0 we expanded to representations in the middle of the beginning block of the U-Net and the middle of the end block of the U-Net. For FLUX.1-dev we simply took one layer in the first third, second third, and last third of the transformer model for both the text and image inputs. 

We also explore different ensembling strategies across our different model architectures to combine the classifiers trained on each individual layer's probes into a single prediction. The four strategies we try are all-concat, soft voting, weight soft-voting, and stacking. All-concat simply concatenates all the probes into one vector and trains a linear classifier on this concatentated vector. Soft voting averages the predicted class probabilities across all per-layer classifiers with equal weight, yielding the simplest and assumption-free combination. Weighted soft voting generalises this by learning a single non-negative weight per layer on a held-out validation split weights are parameterised through a softmax and optimised by Nelder-Mead to minimise the cross-entropy of the weighted-average probabilities, so keys that are individually more informative dominate the ensemble. Stacking instead trains a multinomial logistic-regression meta-learner on the validation set using the concatenated per-layer class-probability vectors as meta-features; this allows the meta-learner to capture interactions between keys rather than only re-scaling them. Overall the best strategy varies between architectures but the differences between them are marginal indicating that there isn't much optimization over which strategy to use needed (Appendix~\ref{app:ensembling_choice_itw}).

\begin{figure}[t]
    \centering
    \includegraphics[width=\linewidth]{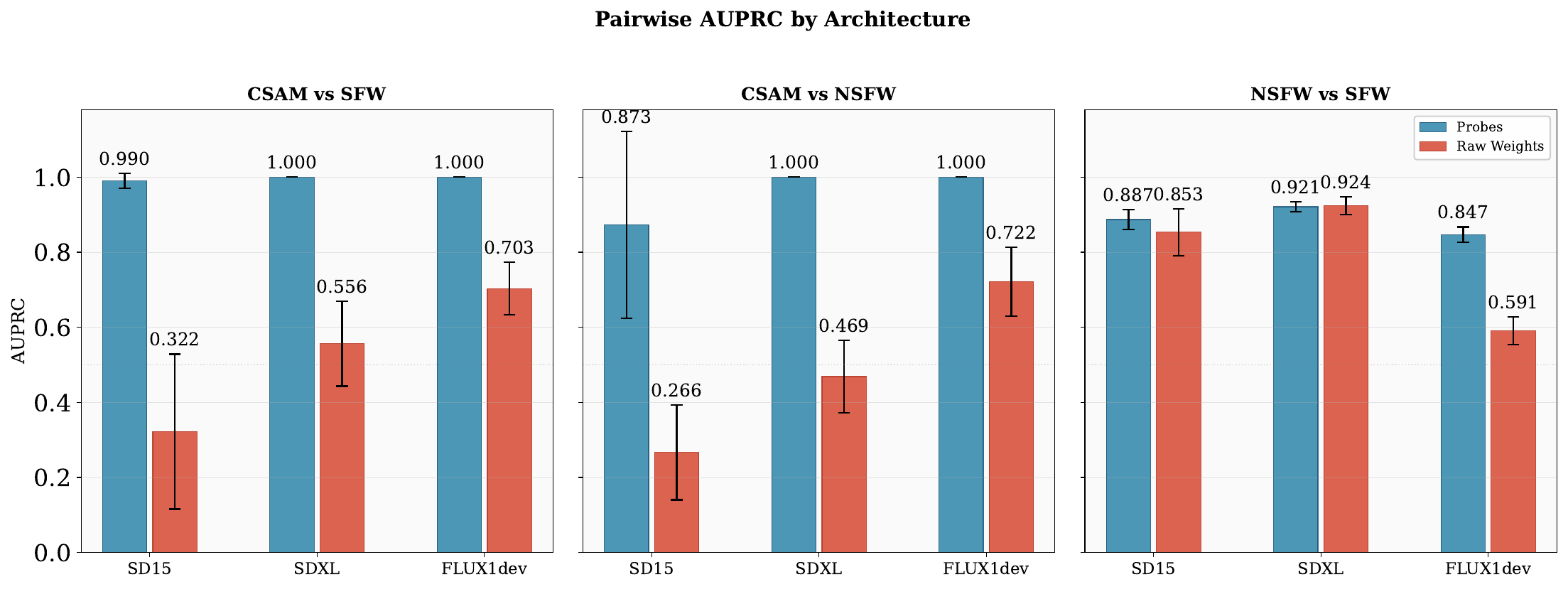}
    \caption{AUPRC across our three models (SD 1.5, SDXL 1.0, and FLUX.1-dev) computed for each pair of classes (SFW vs. CSAM), (NSFW vs. CSAM), and (SFW vs. NSFW). Note that Gaussian probing outperforms raw weights significantly when tasked with distinguishing CSAM from either SFW or NSFW. NSFW vs CSAM should be viewed as the most difficult of the three. The methods perform more comparably when looking at SFW vs NSFW.}
    \label{fig:csam_auprc}
\end{figure}

%% file: sections/discussion.tex
\setlength{\parskip}{0pt}
%Summary of paper's conclusions
Assessing model capabilities and specialization for harmful content has so far focused solely on \textit{generative evaluations}. Generative evaluations capture what models generate under curated prompts, not what they are capable of generating under adversarial or untested conditions. In domains where generation is restricted or unlawful, this creates a measurement gap precisely where risk is highest. In applications such as CSAM detection, output-based evaluation is both limited and legally challenging, motivating alternative approaches to capability assessment. In particular, these applications warrant the need for \textit{non-generative assessments}.
% \ascomment{Also, there is a computational cost of hiring someone to do this!}  \ascomment{What about automated red-teaming? Have people looked at automated approaches to detect CSAM?} 

In this work, we show that analyzing model weights, whether via mechanistic or functional representations, provides a suitable approach for non-generative evaluations of model specialization; a narrower problem within broader capability assessment motivated by open-weight platform governance challenges. This work provides a practical and scalable solution to address this broad governance challenge. We validate our method on our motivating application, which is detecting CSAM specialization in the wild.

Non-generative evaluation is essential in limited and legally challenging applications such as CSAM or NCII detection. In addition, we view it as complementary in harmful content areas where generative evaluations are allowed, such as bioweapons, hate speech, and cyberattack knowledge. Non-generative evaluations can help reduce the need for manual human review of large scale outputs produced in the generative evaluation paradigm, helping to reduce the psychological harms experienced by human reviewers. Our approach, Gaussian probing, supports the ongoing discussions in the weight-space learning community about the importance of mechanistic vs. functional representations. In particular, we show the value of functional representations across scalability/efficiency and robustness. We provide specific criteria that can be useful for future weight-space learning research to use for evaluating new methods for applications where weights are the data modality.

Our work addresses a central open problem in AI child safety regarding evaluation restrictions on AIG-CSAM capabilities. Having demonstrated the effectiveness of different representations we show that non-generative evaluation is possible. We believe that this now sets the stage for future research to continue to study this direction answering questions related to potential vulnerabilities, theoretical underpinnings, and building more robust algorithms. 

Solutions to the evaluation without generation problem are widely applicable for a variety of stakeholders including  open-weight hosting platforms, child safety organizations, and regulatory authorities. Model hosting platforms represent the most immediate application of this approach.  It enables scalable pre-screening of uploaded LoRAs, addressing a key limitation of output-based moderation workflows. This tool can also help improve existing concept tagging systems on these platforms. NSFW labeling on model hosting platforms can be quite noisy, as it is creator-driven. Many SFW LoRAs are often labeled incorrectly and are in fact NSFW specialized. Our tool can help reduce these incorrectly labeled uploads and provide a better user experience. 

At the same time, capability-level filtering raises important questions about moderation and censorship. Any system that flags model capabilities must previously define which capabilities are unacceptable. Setting these boundaries does not just include technical decisions; it also requires discussions around (moral) values, defining which model behaviors should be restricted. Screening systems must therefore balance preventing harm and limiting legitimate uses of generative models. Without clear policies and human oversight, automated filtering could easily drift towards forms of censorship.

Additionally, much more research is required to fully characterize this problem and building a system that is robust to many kinds of adversarial attacks that could evade detection. In this work, we demonstrate one kind of attack via weight rescaling but we encourage further research to understand how other kinds of weight manipulations can evade both our projected weights classifier and our Gaussian probing classifier. We also focus on a narrow version of the evaluation without generation problem which is harmful LoRA finetuning. Further work is needed to determine whether it is possible to answer this question affirmatively using pretrained models alone. In the pretrained model setting, issues of scale and label availability become even more pertinent. 

Overall, this work reframes evaluation as a weight-space auditing problem and demonstrates a scalable approach to pre-deployment screening of harmful specialization. Having answered our original question in the affirmative it lays a foundation for further work in this important research and application area. Such methods may reduce reliance on human review and enable capability assessment in domains where output-based evaluation is constrained or infeasible. In practice, this provides a pathway for platform-level risk screening of fine-tuned models, supporting safer deployment while balancing open innovation with preventive governance.

%% file: references.bib
@article{dravid2024interpreting,
  title={Interpreting the weight space of customized diffusion models},
  author={Dravid, Amil and Gandelsman, Yossi and Wang, Kuan-Chieh and Abdal, Rameen and Wetzstein, Gordon and Efros, Alexei and Aberman, Kfir},
  journal={Advances in Neural Information Processing Systems},
  volume={37},
  pages={137334--137371},
  year={2024}
}

@article{khader2024diffguard,
  title={Diffguard: Text-based safety checker for diffusion models},
  author={Khader, Massine El and Bouzidi, Elias Al and Oumida, Abdellah and Sbaihi, Mohammed and Binard, Eliott and Poli, Jean-Philippe and Ouerdane, Wassila and Addad, Boussad and Kapusta, Katarzyna},
  journal={arXiv preprint arXiv:2412.00064},
  year={2024}
}

@article{li2025diffuguard,
  title={Diffuguard: How intrinsic safety is lost and found in diffusion large language models},
  author={Li, Zherui and Nie, Zheng and Zhou, Zhenhong and Liu, Yue and Zhang, Yitong and Cheng, Yu and Wen, Qingsong and Wang, Kun and Guo, Yufei and Zhang, Jiaheng},
  journal={arXiv preprint arXiv:2509.24296},
  year={2025}
}

@inproceedings{qu2023unsafe,
  title={Unsafe diffusion: On the generation of unsafe images and hateful memes from text-to-image models},
  author={Qu, Yiting and Shen, Xinyue and He, Xinlei and Backes, Michael and Zannettou, Savvas and Zhang, Yang},
  booktitle={Proceedings of the 2023 ACM SIGSAC conference on computer and communications security},
  pages={3403--3417},
  year={2023}
}

@article{luccioni2023stable,
  title={Stable bias: Evaluating societal representations in diffusion models},
  author={Luccioni, Sasha and Akiki, Christopher and Mitchell, Margaret and Jernite, Yacine},
  journal={Advances in Neural Information Processing Systems},
  volume={36},
  pages={56338--56351},
  year={2023}
}

@article{lee2023holistic,
  title={Holistic evaluation of text-to-image models},
  author={Lee, Tony and Yasunaga, Michihiro and Meng, Chenlin and Mai, Yifan and Park, Joon Sung and Gupta, Agrim and Zhang, Yunzhi and Narayanan, Deepak and Teufel, Hannah and Bellagente, Marco and others},
  journal={Advances in Neural Information Processing Systems},
  volume={36},
  pages={69981--70011},
  year={2023}
}

@article{wei2023jailbroken,
  title={Jailbroken: How does llm safety training fail?},
  author={Wei, Alexander and Haghtalab, Nika and Steinhardt, Jacob},
  journal={Advances in neural information processing systems},
  volume={36},
  pages={80079--80110},
  year={2023}
}

@inproceedings{jin2025jailbreakdiffbench,
  title={JailbreakDiffBench: A Comprehensive Benchmark for Jailbreaking Diffusion Models},
  author={Jin, Xiaolong and Weng, Zixuan and Guo, Hanxi and Yin, Chenlong and Cheng, Siyuan and Shen, Guangyu and Zhang, Xiangyu},
  booktitle={Proceedings of the IEEE/CVF International Conference on Computer Vision},
  pages={16461--16471},
  year={2025}
}

@inproceedings{ma2025jailbreaking,
  title={Jailbreaking prompt attack: A controllable adversarial attack against diffusion models},
  author={Ma, Jiachen and Li, Yijiang and Xiao, Zhiqing and Cao, Anda and Zhang, Jie and Ye, Chao and Zhao, Junbo},
  booktitle={Findings of the Association for Computational Linguistics: NAACL 2025},
  pages={3141--3157},
  year={2025}
}

@inproceedings{li2024shake,
  title={Shake to leak: Fine-tuning diffusion models can amplify the generative privacy risk},
  author={Li, Zhangheng and Hong, Junyuan and Li, Bo and Wang, Zhangyang},
  booktitle={2024 IEEE Conference on Secure and Trustworthy Machine Learning (SaTML)},
  pages={18--32},
  year={2024},
  organization={IEEE}
}

@article{truong2025attacks,
  title={Attacks and defenses for generative diffusion models: A comprehensive survey},
  author={Truong, Vu Tuan and Dang, Luan Ba and Le, Long Bao},
  journal={ACM Computing Surveys},
  volume={57},
  number={8},
  pages={1--44},
  year={2025},
  publisher={ACM New York, NY}
}

@article{zheng2023understanding,
  title={Understanding and improving adversarial attacks on latent diffusion model},
  author={Zheng, Boyang and Liang, Chumeng and Wu, Xiaoyu and Liu, Yan},
  year={2023}
}

@article{rando2022red,
  title={Red-teaming the stable diffusion safety filter},
  author={Rando, Javier and Paleka, Daniel and Lindner, David and Heim, Lennart and Tram{\`e}r, Florian},
  journal={arXiv preprint arXiv:2210.04610},
  year={2022}
}

@article{fan2023salun,
  title={Salun: Empowering machine unlearning via gradient-based weight saliency in both image classification and generation},
  author={Fan, Chongyu and Liu, Jiancheng and Zhang, Yihua and Wong, Eric and Wei, Dennis and Liu, Sijia},
  journal={arXiv preprint arXiv:2310.12508},
  year={2023}
}

@inproceedings{wu2024scissorhands,
  title={Scissorhands: Scrub data influence via connection sensitivity in networks},
  author={Wu, Jing and Harandi, Mehrtash},
  booktitle={European Conference on Computer Vision},
  pages={367--384},
  year={2024},
  organization={Springer}
}

@article{wu2024erasediff,
  title={Erasediff: Erasing data influence in diffusion models},
  author={Wu, Jing and Le, Trung and Hayat, Munawar and Harandi, Mehrtash},
  journal={arXiv preprint arXiv:2401.05779},
  year={2024}
}

@inproceedings{lu2024mace,
  title={Mace: Mass concept erasure in diffusion models},
  author={Lu, Shilin and Wang, Zilan and Li, Leyang and Liu, Yanzhu and Kong, Adams Wai-Kin},
  booktitle={Proceedings of the IEEE/CVF Conference on Computer Vision and Pattern Recognition},
  pages={6430--6440},
  year={2024}
}

@misc{zhang2024defensiveunlearningadversarialtraining,
  title={Defensive Unlearning with Adversarial Training for Robust Concept Erasure in Diffusion Models},
  author={Yimeng Zhang and Xin Chen and Jinghan Jia and Yihua Zhang and Chongyu Fan and Jiancheng Liu and Mingyi Hong and Ke Ding and Sijia Liu},
  year={2024},
  eprint={2405.15234},
  archivePrefix={arXiv},
  primaryClass={cs.CV},
  url={https://arxiv.org/abs/2405.15234},
}

@misc{huang2024recelerreliableconcepterasing,
  title={Receler: Reliable Concept Erasing of Text-to-Image Diffusion Models via Lightweight Erasers},
  author={Chi-Pin Huang and Kai-Po Chang and Chung-Ting Tsai and Yung-Hsuan Lai and Fu-En Yang and Yu-Chiang Frank Wang},
  year={2024},
  eprint={2311.17717},
  archivePrefix={arXiv},
  primaryClass={cs.CV},
  url={https://arxiv.org/abs/2311.17717},
}

@inproceedings{lin2014microsoft,
  title={Microsoft coco: Common objects in context},
  author={Lin, Tsung-Yi and Maire, Michael and Belongie, Serge and Hays, James and Perona, Pietro and Ramanan, Deva and Doll{\'a}r, Piotr and Zitnick, C Lawrence},
  booktitle={Computer Vision--ECCV 2014: 13th European Conference, Zurich, Switzerland, September 6-12, 2014, Proceedings, Part V 13},
  pages={740--755},
  year={2014},
  organization={Springer}
}

@misc{von-platen-etal-2022-diffusers,
author = {Patrick von Platen and Suraj Patil and Anton Lozhkov and Pedro Cuenca and Nathan Lambert and Kashif Rasul and Mishig Davaadorj and Dhruv Nair and Sayak Paul and William Berman and Yiyi Xu and Steven Liu and Thomas Wolf},
title = {Diffusers: State-of-the-art diffusion models},
year = {2022},
publisher = {GitHub},
journal = {GitHub repository}
}

@article{fuchi2024erasing,
  title={Erasing Concepts from Text-to-Image Diffusion Models with Few-shot Unlearning},
  author={Fuchi, Masane and Takagi, Tomohiro},
  journal={arXiv preprint arXiv:2405.07288},
  year={2024}
}

@inproceedings{gandikota2024unified,
  title={Unified concept editing in diffusion models},
  author={Gandikota, Rohit and Orgad, Hadas and Belinkov, Yonatan and Materzy{\'n}ska, Joanna and Bau, David},
  booktitle={Proceedings of the IEEE/CVF Winter Conference on Applications of Computer Vision},
  pages={5111--5120},
  year={2024}
}

@inproceedings{yang2024sneakyprompt,
  title={Sneakyprompt: Jailbreaking text-to-image generative models},
  author={Yang, Yuchen and Hui, Bo and Yuan, Haolin and Gong, Neil and Cao, Yinzhi},
  booktitle={2024 IEEE symposium on security and privacy (SP)},
  pages={897--912},
  year={2024},
  organization={IEEE}
}

@article{ko2024boosting,
  title={Boosting alignment for post-unlearning text-to-image generative models},
  author={Ko, Myeongseob and Li, Henry and Wang, Zhun and Patsenker, Jonathan and Wang, Jiachen Tianhao and Li, Qinbin and Jin, Ming and Song, Dawn and Jia, Ruoxi},
  journal={Advances in Neural Information Processing Systems},
  volume={37},
  pages={85131--85154},
  year={2024}
}

@inproceedings{gandikota2023erasing,
  title={Erasing concepts from diffusion models},
  author={Gandikota, Rohit and Materzynska, Joanna and Fiotto-Kaufman, Jaden and Bau, David},
  booktitle={Proceedings of the IEEE/CVF International Conference on Computer Vision},
  pages={2426--2436},
  year={2023}
}

@inproceedings{suriyakumar2024unstable,
  title={Unstable unlearning: The hidden risk of concept resurgence in diffusion models},
  author={Suriyakumar, Vinith Menon and Alur, Rohan and Sekhari, Ayush and Raghavan, Manish and Wilson, Ashia C},
  booktitle={ICLR 2025 Workshop on Navigating and Addressing Data Problems for Foundation Models},
  year={2024}
}

@inproceedings{gao2025meta,
  title={Meta-unlearning on diffusion models: Preventing relearning unlearned concepts},
  author={Gao, Hongcheng and Pang, Tianyu and Du, Chao and Hu, Taihang and Deng, Zhijie and Lin, Min},
  booktitle={Proceedings of the IEEE/CVF International Conference on Computer Vision},
  pages={2131--2141},
  year={2025}
}

@article{cretu2025evaluating,
  title={Evaluating Concept Filtering Defenses against Child Sexual Abuse Material Generation by Text-to-Image Models},
  author={Cretu, Ana-Maria and Kireev, Klim and Abdalla, Amro and Obinna, Wisdom and Meier, Raphael and Bargal, Sarah Adel and Redmiles, Elissa M and Troncoso, Carmela},
  journal={arXiv preprint arXiv:2512.05707},
  year={2025}
}

@article{lu2025concepts,
  title={When Are Concepts Erased From Diffusion Models?},
  author={Lu, Kevin and Kriplani, Nicky and Gandikota, Rohit and Pham, Minh and Bau, David and Hegde, Chinmay and Cohen, Niv},
  journal={arXiv preprint arXiv:2505.17013},
  year={2025}
}

@misc{kohya_ss,
  author = {bmaltais},
  title = {kohya\_ss: GUI and CLI for training diffusion models},
  year = {2025},
  howpublished = {\url{https://github.com/bmaltais/kohya_ss}},
  note = {Accessed: 2026-04-23}
}

@article{kwon2022diffusion,
  title={Diffusion models already have a semantic latent space},
  author={Kwon, Mingi and Jeong, Jaeseok and Uh, Youngjung},
  journal={arXiv preprint arXiv:2210.10960},
  year={2022}
}

@misc{invokeai,
  title = {InvokeAI},
  author = {{Invoke AI Contributors}},
  year = {2026},
  url = {https://github.com/invoke-ai/InvokeAI}
}

@article{hu2022lora,
  title={Lora: Low-rank adaptation of large language models.},
  author={Hu, Edward J and Shen, Yelong and Wallis, Phillip and Allen-Zhu, Zeyuan and Li, Yuanzhi and Wang, Shean and Wang, Lu and Chen, Weizhu and others},
  journal={ICLR},
  volume={1},
  number={2},
  pages={3},
  year={2022}
}

@book{gillespie2018custodians,
  title={Custodians of the Internet: Platforms, content moderation, and the hidden decisions that shape social media},
  author={Gillespie, Tarleton},
  year={2018},
  publisher={Yale University Press}
}

@inproceedings{steiger2021psychological,
  title={The psychological well-being of content moderators: the emotional labor of commercial moderation and avenues for improving support},
  author={Steiger, Miriah and Bharucha, Timir J and Venkatagiri, Sukrit and Riedl, Martin J and Lease, Matthew},
  booktitle={Proceedings of the 2021 CHI conference on human factors in computing systems},
  pages={1--14},
  year={2021}
}

@book{roberts2019behind,
  title={Behind the screen},
  author={Roberts, Sarah T},
  year={2019},
  publisher={Yale University Press}
}

@article{ganguli2022red,
  title={Red teaming language models to reduce harms: Methods, scaling behaviors, and lessons learned},
  author={Ganguli, Deep and Lovitt, Liane and Kernion, Jackson and Askell, Amanda and Bai, Yuntao and Kadavath, Saurav and Mann, Ben and Perez, Ethan and Schiefer, Nicholas and Ndousse, Kamal and others},
  journal={arXiv preprint arXiv:2209.07858},
  year={2022}
}

@misc{shevlane2023modelevaluationextremerisks,
      title={Model evaluation for extreme risks}, 
      author={Toby Shevlane and Sebastian Farquhar and Ben Garfinkel and Mary Phuong and Jess Whittlestone and Jade Leung and Daniel Kokotajlo and Nahema Marchal and Markus Anderljung and Noam Kolt and Lewis Ho and Divya Siddarth and Shahar Avin and Will Hawkins and Been Kim and Iason Gabriel and Vijay Bolina and Jack Clark and Yoshua Bengio and Paul Christiano and Allan Dafoe},
      year={2023},
      eprint={2305.15324},
      archivePrefix={arXiv},
      primaryClass={cs.AI},
      url={https://arxiv.org/abs/2305.15324}, 
}

@misc{amaye2025objectsegmentation,
  author       = {Mayes, Andrew},
  title        = {Object Segmentation Dataset},
  year         = {2025},
  howpublished = {\url{https://huggingface.co/datasets/amaye15/object-segmentation}},
  note         = {Accessed: 2024-04-08}
}

@misc{nyanko2023danbooru,
  author       = {Nyanko},
  title        = {Danbooru2023},
  year         = {2023},
  howpublished = {\url{https://huggingface.co/datasets/nyanko7/danbooru2023}},
  note         = {Accessed: 2024-04-08}
}

@article{fuest2026diffusion,
  title={Diffusion models and representation learning: A survey},
  author={Fuest, Michael and Ma, Pingchuan and Gui, Ming and Schusterbauer, Johannes and Hu, Vincent Tao and Ommer, Bj{\"o}rn},
  journal={IEEE Transactions on Pattern Analysis and Machine Intelligence},
  year={2026},
  publisher={IEEE}
}

@article{yang2023diffusion,
  title={Diffusion models: A comprehensive survey of methods and applications},
  author={Yang, Ling and Zhang, Zhilong and Song, Yang and Hong, Shenda and Xu, Runsheng and Zhao, Yue and Zhang, Wentao and Cui, Bin and Yang, Ming-Hsuan},
  journal={ACM computing surveys},
  volume={56},
  number={4},
  pages={1--39},
  year={2023},
  publisher={ACM New York, NY, USA}
}

@article{conmy2023towards,
  title={Towards automated circuit discovery for mechanistic interpretability},
  author={Conmy, Arthur and Mavor-Parker, Augustine and Lynch, Aengus and Heimersheim, Stefan and Garriga-Alonso, Adri{\`a}},
  journal={Advances in Neural Information Processing Systems},
  volume={36},
  pages={16318--16352},
  year={2023}
}

@misc{bereska2024mechanisticinterpretabilityaisafety,
      title={Mechanistic Interpretability for AI Safety -- A Review}, 
      author={Leonard Bereska and Efstratios Gavves},
      year={2024},
      eprint={2404.14082},
      archivePrefix={arXiv},
      primaryClass={cs.AI},
      url={https://arxiv.org/abs/2404.14082}, 
}

@misc{morelli2016nsfw,
  title={OpenNSFW: A Dataset for NSFW Image Classification},
  author={Morelli, Andres and others},
  year={2016},
  howpublished={Yahoo open source project},
  url={https://github.com/yahoo/open_nsfw}
}

@misc{unsplashlite,
  title={Unsplash Lite Dataset},
  author={Unsplash},
  year={2017},
  howpublished={\url{https://github.com/unsplash/datasets}}
}

@misc{saleh2015largescaleclassificationfineartpaintings,
      title={Large-scale Classification of Fine-Art Paintings: Learning The Right Metric on The Right Feature}, 
      author={Babak Saleh and Ahmed Elgammal},
      year={2015},
      eprint={1505.00855},
      archivePrefix={arXiv},
      primaryClass={cs.CV},
      url={https://arxiv.org/abs/1505.00855}, 
}

@article{schuhmann2022laion,
  title={Laion-5b: An open large-scale dataset for training next generation image-text models},
  author={Schuhmann, Christoph and Beaumont, Romain and Vencu, Richard and Gordon, Cade and Wightman, Ross and Cherti, Mehdi and Coombes, Theo and Katta, Aarush and Mullis, Clayton and Wortsman, Mitchell and others},
  journal={Advances in neural information processing systems},
  volume={35},
  pages={25278--25294},
  year={2022}
}

@article{kuznetsova2020open,
  title={The open images dataset v4: Unified image classification, object detection, and visual relationship detection at scale},
  author={Kuznetsova, Alina and Rom, Hassan and Alldrin, Neil and Uijlings, Jasper and Krasin, Ivan and Pont-Tuset, Jordi and Kamali, Shahab and Popov, Stefan and Malloci, Matteo and Kolesnikov, Alexander and others},
  journal={International journal of computer vision},
  volume={128},
  number={7},
  pages={1956--1981},
  year={2020},
  publisher={Springer}
}

@inproceedings{changpinyo2021conceptual,
  title={Conceptual 12m: Pushing web-scale image-text pre-training to recognize long-tail visual concepts},
  author={Changpinyo, Soravit and Sharma, Piyush and Ding, Nan and Soricut, Radu},
  booktitle={Proceedings of the IEEE/CVF conference on computer vision and pattern recognition},
  pages={3558--3568},
  year={2021}
}

@inproceedings{plummer2015flickr30k,
  title={Flickr30k entities: Collecting region-to-phrase correspondences for richer image-to-sentence models},
  author={Plummer, Bryan A and Wang, Liwei and Cervantes, Chris M and Caicedo, Juan C and Hockenmaier, Julia and Lazebnik, Svetlana},
  booktitle={Proceedings of the IEEE international conference on computer vision},
  pages={2641--2649},
  year={2015}
}

@misc{iwf2026aigcsamupdate,
  author       = {{Internet Watch Foundation (IWF)}},
  title        = {Harm without Limits: AI child sexual abuse material through the eyes of our Analysts},
  institution  = {Internet Watch Foundation},
  month        = mar,
  year         = {2026},
  url          = {https://www.iwf.org.uk/media/hl1nvdti/iwf-ai-csam-report-2026.pdf}
}

@misc{flux2024,
    author={Black Forest Labs},
    title={FLUX},
    year={2024},
    howpublished={\url{https://github.com/black-forest-labs/flux}},
}

@article{wan2025wan,
  title={Wan: Open and advanced large-scale video generative models},
  author={Wan, Team and Wang, Ang and Ai, Baole and Wen, Bin and Mao, Chaojie and Xie, Chen-Wei and Chen, Di and Yu, Feiwu and Zhao, Haiming and Yang, Jianxiao and others},
  journal={arXiv preprint arXiv:2503.20314},
  year={2025}
}

@inproceedings{rombach2022high,
  title={High-resolution image synthesis with latent diffusion models},
  author={Rombach, Robin and Blattmann, Andreas and Lorenz, Dominik and Esser, Patrick and Ommer, Bj{\"o}rn},
  booktitle={Proceedings of the IEEE/CVF conference on computer vision and pattern recognition},
  pages={10684--10695},
  year={2022}
}

@inproceedings{schurholt2024towards,
  title={Towards scalable and versatile weight space learning},
  author={Sch{\"u}rholt, Konstantin and Mahoney, Michael W and Borth, Damian},
  booktitle={Proceedings of the 41st International Conference on Machine Learning},
  pages={43947--43966},
  year={2024}
}

@misc{unterthiner2021predictingneuralnetworkaccuracy,
      title={Predicting Neural Network Accuracy from Weights}, 
      author={Thomas Unterthiner and Daniel Keysers and Sylvain Gelly and Olivier Bousquet and Ilya Tolstikhin},
      year={2021},
      eprint={2002.11448},
      archivePrefix={arXiv},
      primaryClass={stat.ML},
      url={https://arxiv.org/abs/2002.11448}, 
}

@misc{kahana2024deeplinearprobegenerators,
      title={Deep Linear Probe Generators for Weight Space Learning}, 
      author={Jonathan Kahana and Eliahu Horwitz and Imri Shuval and Yedid Hoshen},
      year={2024},
      eprint={2410.10811},
      archivePrefix={arXiv},
      primaryClass={cs.LG},
      url={https://arxiv.org/abs/2410.10811}, 
}

@inproceedings{Schurholt2021,
 author = {Sch\"{u}rholt, Konstantin and Kostadinov, Dimche and Borth, Damian},
 booktitle = {Advances in Neural Information Processing Systems},
 editor = {M. Ranzato and A. Beygelzimer and Y. Dauphin and P.S. Liang and J. Wortman Vaughan},
 pages = {16481--16493},
 publisher = {Curran Associates, Inc.},
 title = {Self-Supervised Representation Learning on Neural Network Weights for Model Characteristic Prediction},
 url = {https://proceedings.neurips.cc/paper_files/paper/2021/file/89562dccfeb1d0394b9ae7e09544dc70-Paper.pdf},
 volume = {34},
 year = {2021}
}

@article{thiel2023generative,
  title={Generative ML and CSAM: Implications and mitigations},
  author={Thiel, David and Stroebel, Melissa and Portnoff, Rebecca},
  journal={Stanford Digital Repository},
  year={2023}
}

@misc{thornsafetybydesign,
title={Safety by Design for Generative AI: Preventing Child Sexual Abuse},
howpublished={\url{https://info.thorn.org/hubfs/thorn-safety-by-design-for-generative-AI.pdf}},
author={Thorn and {All Tech is Human}},
year=2024
}

@misc{ncmec2025cybertipline,
  title={2024 CyberTipline Report},
  author={National Center for Missing \& Exploited Children},
  year={2025},
  howpublished={\url{https://www.missingkids.org/content/dam/missingkids/pdfs/cybertiplinedata2024/2024-CyberTipline-Report.pdf}},
  publisher={National Center for Missing \& Exploited Children (NCMEC)}
}

@misc{thornprogressreport,
  title={Safety by Design Three-Month Progress Report \#3: November 2024 to January 2025},
  author={Thorn},
  howpublished="\url{https://info.thorn.org/hubfs/Thorn-SafetybyDesign-ThreeMonthProgressReport-3.pdf?}",
  year={2025},
  publisher={Thorn}
}

@inproceedings{yuan2025lurks,
  title={What Lurks Within? Concept Auditing for Shared Diffusion Models at Scale},
  author={Yuan, Xiaoyong and Ma, Xiaolong and Guo, Linke and Zhang, Lan},
  booktitle={Proceedings of the 2025 ACM SIGSAC Conference on Computer and Communications Security},
  pages={4154--4168},
  year={2025}
}

@article{openproblemschildsafety2026,
  title={Child Safety Necessitates New Approaches to AI Safety},
  author={Kale*, Neil and Portnoff*, Rebecca and Thaker, Pratiksha and Simpson, Michael and Wang, Robertson and Kuo, Kevin and Yadav, Chhavi and Smith, Virginia},
  year={2026},
  url={https://papers.ssrn.com/sol3/papers.cfm?abstract_id=6206860},
  note={* Equal contribution}
}

@INPROCEEDINGS{lyu2024,
  author={Lyu, Mengyao and Yang, Yuhong and Hong, Haiwen and Chen, Hui and Jin, Xuan and He, Yuan and Xue, Hui and Han, Jungong and Ding, Guiguang},
  booktitle={2024 IEEE/CVF Conference on Computer Vision and Pattern Recognition (CVPR)}, 
  title={One-dimensional Adapter to Rule Them All: Concepts, Diffusion Models and Erasing Applications}, 
  year={2024},
  volume={},
  number={},
  pages={7559-7568},
  doi={10.1109/CVPR52733.2024.00722}}

@misc{civitaiSPM,
	author = {Civitai},
	title = {{I}ntroducing {C}ivitai {G}reen: {O}ur {C}ontinued {C}ommitment to {C}ommunity {S}afety | {C}ivitai --- civitai.com},
	howpublished = {\url{https://civitai.com/articles/7078/introducing-civitai-green-our-continued-commitment-to-community-safety}},
	year = {2024},
	note = {[Accessed 21-04-2026]},
}
